\documentclass[]{article}
\usepackage{amsmath}
\usepackage{graphicx} 
\usepackage{amsfonts}
\usepackage{enumerate}
\usepackage[toc,page]{appendix}

\newtheorem{theorem}{Theorem}[section]
\newtheorem{lemma}[theorem]{Lemma}
\newtheorem{proposition}[theorem]{Proposition}
\newtheorem{corollary}[theorem]{Corollary}
\newtheorem{remark}{Remark}
\newtheorem{proof}{Proof}
\newcommand\smul{\cdot}

\title{Characterizing Ambiguity in Light Source Invariant Shape from Shading\thanks{This work was
supported by the National Science Foundation.}} 

\author{Benjamin Kunsberg,  Steven W. Zucker \thanks{
\text{benjamin.kunsberg@yale.edu}. Questions, comments, or corrections
to this document may be directed to that email address.}}

\begin{document}
\maketitle
\newcommand{\slugmaster}{%
\slugger{siims}{xxxx}{xx}{x}{x--x}}

\begin{abstract}
Shape from shading is a classical inverse problem in computer vision.  This shape reconstruction problem is inherently ill-defined;  it depends on the assumed light source direction.  We introduce a novel mathematical formulation for calculating local surface shape based on covariant derivatives of the shading flow field, rather than the customary integral minimization or P.D.E approaches.  On smooth surfaces, we show second derivatives of brightness are independent of the light sources and can be directly related to surface properties.  We use these measurements to define the matching local family of surfaces that can result from any given shading patch, changing the emphasis to characterizing ambiguity in the problem. We give an example of how these local surface ambiguities collapse along certain image contours and how this can be used for the reconstruction problem.

\end{abstract}


\section{Introduction}

\begin{figure}[t]
\begin{center}
\includegraphics[width= 1 \linewidth]{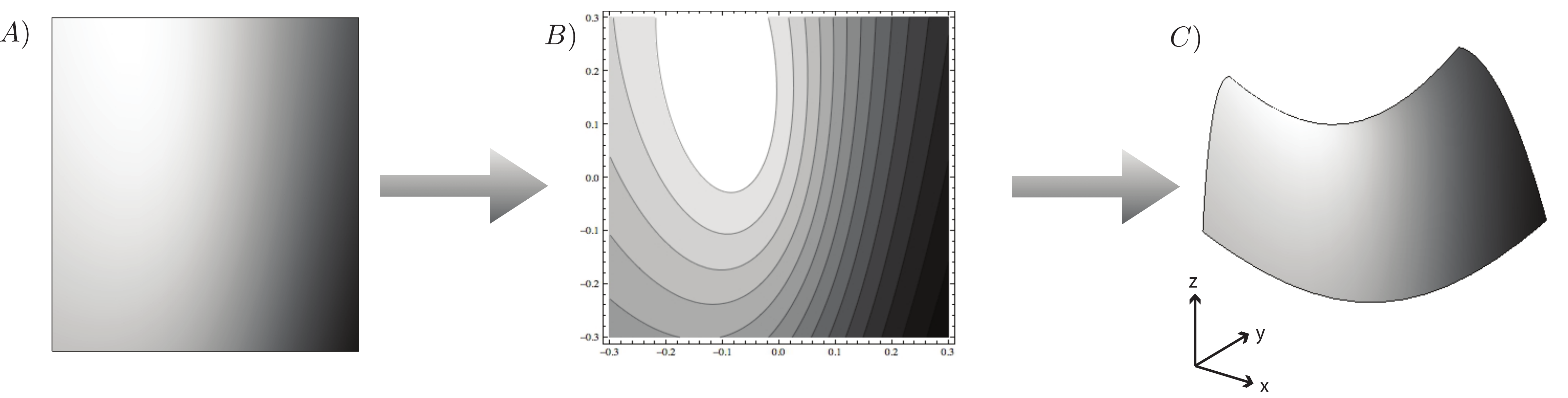}
\caption{Classical methods attempt to go from pixel values of the image (A) to the surface (C).  In this work, we use the intermediate representation of the isophote structure (B).  This is both supported by biological mechanisms and allows us to use the mathematical machinery of vector fields.}
\end{center}
\end{figure}

	In 1866, Ernst Mach \cite{mach} formulated the image irradiance equation to study which surfaces humans perceive from different ``light surfaces." He recognized the difficulty in solving this equation ( ``many curved surfaces may correspond to one light surface even if they are illuminated in the same manner") so he focused on studying cylinders.  In analogous fashion, the modern shape from shading community has attempted to resolve this ambiguity with the use of priors on the surface \cite{Barron13}, light source direction(s) \cite{Freeman94, Barron13} and albedo \cite{Barron13}.  However, there is an alternative:  rather than attempting to resolve the ambiguity immediately using priors, one may parametrize the ambiguity and let other cues such as the apparent boundary, highlight lines or cusps, resolve them.  Understanding such ambiguity is the focus of this paper.
	
	To define this ambiguity, we need to derive image properties that are invariant to the direction of the light source.   We will prove that ordinary second derivatives of the image irradiance do not depend on the direction(s) of the light source(s), but rather only on the local surface derivatives and other image derivatives.   The image derivatives can then be used to restrict the potential surfaces corresponding to a local shading patch, regardless of the light source.  We effectively ``cancel out" the light source from the problem.
	In summary, the visible interaction (image data) of light source and surface shape can be separated into two different types of information.  One of these types is dependent on both the light source and surface shape and thus is not useful for a shape from shading algorithm that does not assume or even represent the light source(s).   However, we have isolated components that are only dependent on surface shape and other image properties; thus, measurement of these components can be used to solve for surface shape independent of the light source.

\subsection{Neurobiologial motivation}

\begin{figure}[t]
\begin{center}
\includegraphics[width= 1 \linewidth]{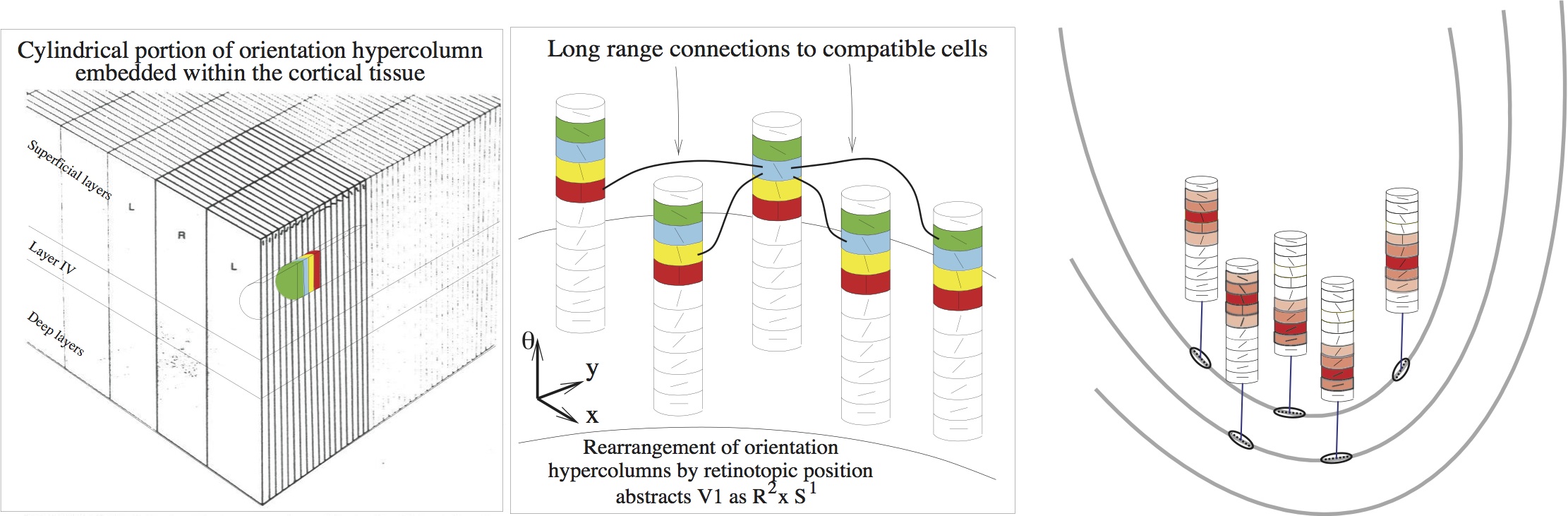}
\caption{V1 mechanisms applied to the isophote curves result in a shading flow field.
(left) Visual cortex contains neurons selective to the local orientation of image patches. In a shading gradient these
will respond most strongly to the local isophote orientation; ie, it's tangent.  A tangential
penetration across V1 yields a column of cells.
(middle) Abstracting the column of cells over a given (retinotopic) position spans all orientations; different orientations at nearby positions regularize the tangent map to the isophotes and reduce noise in local measurements.
(right) Illustrating the lift of the isophones into a V1 style representation. The mathematical analysis in this paper
extends this type of representation to the surface inference prolbem. As such it could be implemented by similar
cortical machinery in higher visual areas.}
\label{fig:v1_mechanisms}
\end{center}
\end{figure}

	Although Mach studied the retina, he did not predict the revolution in neurobiology that revealed how visual cortex is
organized around orientation. Due to the orientation-tuned cells of the V1 area of visual cortex \cite{Hubel88}, we focus on understanding how sets of orientations could correspond to surfaces.  This suggests what at first glance appears to be a small problem change: instead of seeking the map from images to surfaces (and light sources), the image should
first be lifted into an orientation-based representation.  This lift can be accomplished by considering
the image isophotes (\cite{Koenderink90, Koenderink80}); the lift is then the tangent map to these isophotes. This has arisen earlier in computer vision, and is called the shading flow field \cite{zucker}. A significant body of evidence is accumulating that such orientation based representations underly the perception of shape \cite{fleming}, but to our knowledge no one has previously formulated the surface inference problem from it. 

Since the shading flow field could be computed in V1 (Figure \ref{fig:v1_mechanisms}), we are developing a new approach to shape from shading
that is built directly on the information available in visual cortex.  Our computations could be implemented by a combination of feedforward and feedback projections,
supplemented with the long-range horizontal connections within each visual area.  Here we concentrate on developing the math.
The calculations are derived in differential-geometric terms, and a crisp curvature structure emerges from the
transport equations. As such it serves as the foundation of a model for understanding feedforward connections to
higher levels (surfaces) from lower levels (flows). See Fig.~\ref{fig:surface-sff}.

\subsection{Overview of Approach}
We summarize our approach in Fig. \ref{fig:big_pic}. Rather than attempting to infer a surface directly from the image and (e.g.) global light source priors, we think of the surface as a composite of local patches (charts). Each patch is described by its (patch of) shading flow, each of which implies a space of (surface patch, light source) pairs. Much of the formal content of this paper is a way to calculate them.

The possible local surface patches define a ``fibre'' for each patch coordinate; over the surface these fibres form a bundle.
Conceptually, then, the shape-from-shading problem amounts to finding a section through this bundle.  Once a section is obtained, the light source positions can be calculated directly.


There are several advantages to this approach. Ambiguity now is a measure on these fibres, and it can be reduced by certain (local) conditions, for example
curvature at the boundary \cite{Wagemans10}.  Thus it is consistent with Marr's Principle of Least Committment.
Second, the light source positions are essentially an emergent property rather than a prior assumption.
And finally, it mirrors the composite nature of visual inverse problems: there are those confingurations in which solutions are nicely defined, and there
are others that remain inherently ambiguous. A powerful illustration of this is provided by artists' drawings: a single new stroke may change the impression completely,
such as the indication of a highlight, while others may be lost in the cross-hatching of shading. Several cases are discussed at the end of this paper; and analyzed fully in a companion paper.

\begin{figure}[t]
\begin{center}
\includegraphics[height = 3.5in, width = 3in]{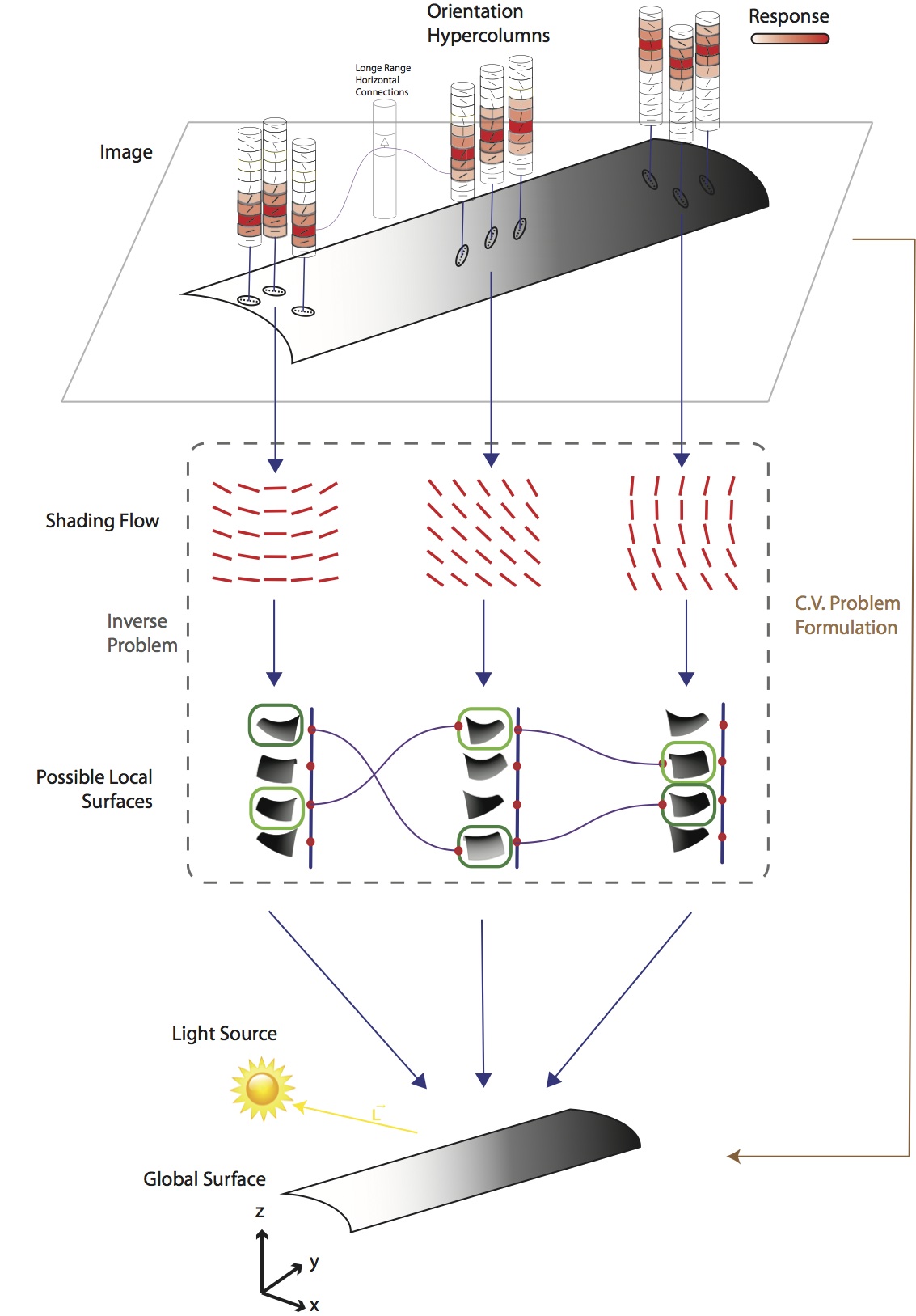}
\caption{This figure illustrates the ``big pictureÓ for our approach to shape from shading. Instead of inferring surfaces directly
from the image, we impose two fibre bundles between them. The first is the lift of the image into the shading flow, and
the second defines the fibre of possible surface patches that are consistent with a given patch of shading flow.
We do not assume a light source position, but instead will use assumptions on certain local features
to restrict the ambiguity. This amounts to finding a section through the bundle of possible local surfaces. The light source
position(s) are then an emergent property.}
\label{fig:big_pic}
\end{center}
\end{figure}

\subsection{Comparisons to Previous Work}

A major difference between our work and previous approaches is not to determine conditions when a P.D.E. system is well-posed \cite{Dieft81, Dupuis94, Oliensis91}, (which can be very difficult for the SFS problem, e.g. see viscosity solutions in \cite{Prados04, Lions93, Rouy92}) , or to argue which assumptions to make.  Rather, we determine how much ambiguity there is when the light source(s) are completely unknown.  Specifically, we investigate which local (surface, light source) pairs are consistent with the shading flow or image information.  This is a purely mathematical problem, with no judgements needed for a prior or other assumptions.

Horn's original algorithm \cite{horn85} solved a first order partial differential equation.  It required a known reflectance map (or very strong conditions on it) and known normals along some closed surface curve.  Under constant albedo, the known reflectance map is equivalent to a known light source direction.  In contrast, our method requires none of this information \emph{a priori}.  

We emphasize that Horn's (and Mach's) method calculates characteristic curves independently.  However, there is information to be exploited among nearby characteristic curves; thus, a method that calculates local surface shape over a 2D image patch (such as ours)  will need less initial information.  This insight is key to the
integral minimization approaches (e.g. \cite{Forsyth11, faugeras, Brooks86, Ikeuchi81}). These are often supplemented with a geometric regularization term or require a unit area assumption \cite{Forsyth11}. (We require no such assumption.) For more information on the various approaches, see \cite{zhang99}.

While integral approaches work over an area, they still directly analyze the image. However, using the shading flow as an intermediate representation provides a very different coordinate system for regularization.  But in addition to the biological motivation for using the shading flows, there is also very useful mathematical machinery  that can be applied once we consider the data as a vector field, rather than just pixel values.  Therefore, we will think of shape from shading as a map from vector fields onto ``shape space."

 Our work extends Pentland \cite{pentland84}.  He classified surfaces in general categories (plane, cylinder, saddle, etc.) based on the signs of the second derivatives of the intensity at a point.  However, this is a broad classification; there is clearly more information to be used than just the sign.   In this paper, we also use information contained in the second derivatives of intensity, although we prove the exact correspondence to the surface curvatures.

Much psychophysical work in SFS is in the spirit of Pentland \cite{pentland84}.  Modern researchers use very simplified stimuli \cite{Sun12, Wagemans10}, although some classical papers question the need for light source assumptions \cite{Mingolla86}. The integration between boundary information and shading is also classical \cite{Shapley85, Ramachandran88}, but understanding why this is perceptually salient follows from our analysis.

Belhumeur et al. (\cite{Bel98,Bel99}) provides another point of reference for our approach.  Consider the space $K$ of all possible Lambertian shading images resulting from all possible smooth surfaces with all possible light source directions.  Fixing a particular surface, a cross section of all possible images corresponding to a particular surface as we vary light source directions defines the illumination cone.  It can be calculated with at most 5 different images of the same surface.  However, there is still the problem of going from illumination cone to surface.  Although this analysis is useful, it does not lead directly to a shape from shading algorithm.

Alternatively, one could fix a single image $I$ and consider the cross section of $K$ defined by that image.  This cross section consists of all possible surfaces and light source pairs that would have resulted in $I$.  We argue that this cross section of $K$ is the more useful one to consider, as our initial data is an image rather than a surface.  Our goal is to define that cross section by using the shading flow field.

In closing the introduction, we make two general observations. As we will show next, the surface isophote tangent at a point is solely dependent on the light source direction and the local Taylor approximation for the surface (up to second order).  For shape from shading, one needs to combine information from multiple isophotes; that is, one needs to implicitly or explicitly calculate some type of derivative of these isophote tangents.  Thus, we believe the correct viewpoint to solve shape from shading should be at least third-order.  

\begin{figure}[!ht]
\begin{center}
\includegraphics[width = 1.5in, height = 1.5in, keepaspectratio]{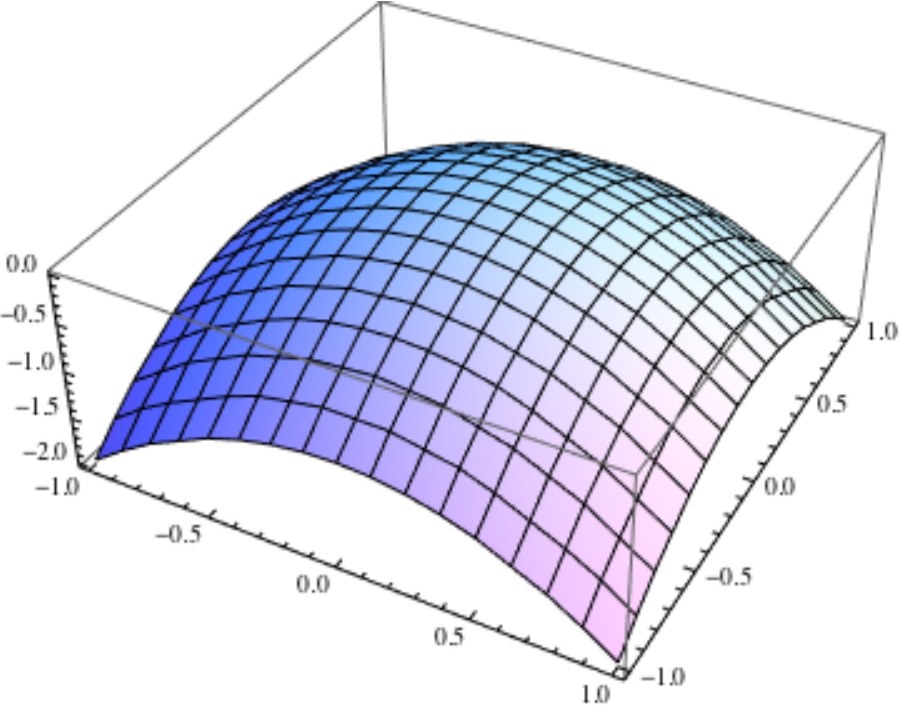} 
\includegraphics[height=1.5in,width=1.5in, keepaspectratio]{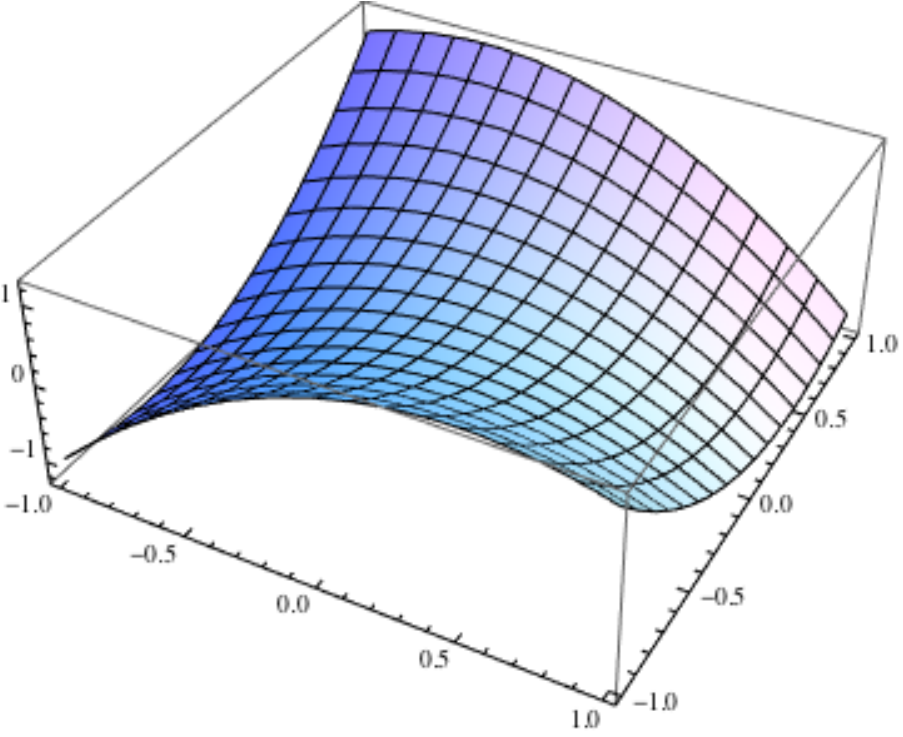} 
\includegraphics[height=1.5in,width=1.5in, keepaspectratio]{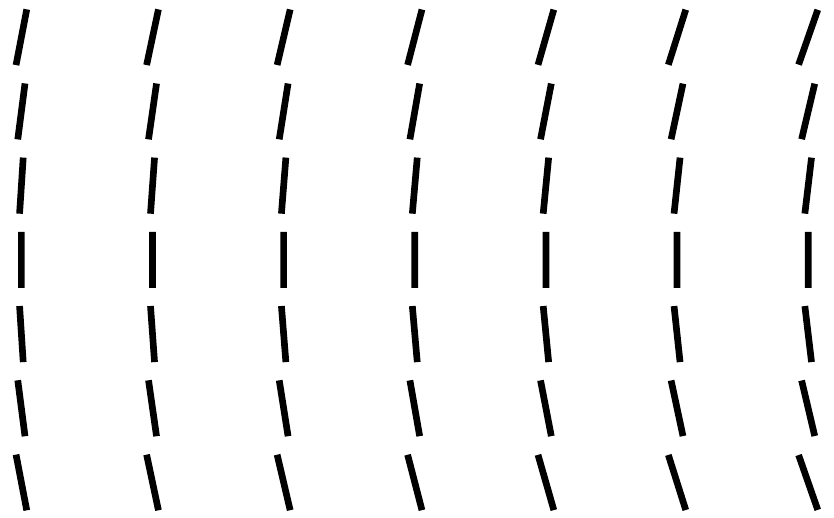} 
\caption{The shading flow field abstracts the image isophotes into a vector field. This has a biological basis (shading flow could be represented in V1 and V2 \cite{Hubel88}; surfaces in V4 and IT \cite{Connor08}). We base the surface inference problem on it, and not
on the raw image. Here we illustrate that
two surfaces (in this instance, a sphere and a saddle) can correspond to the same shading flow, provided the light source changes appropriately.  For the sphere, we need a light source at $(1, 0, 1)$; for the saddle, we need a light source at  $(1, 1, 1)$. Note that this ambiguity extends the convex/concave ambiguity \cite{Erens93, pentland84} normally considered in computer vision. It can be proved that, for the problem considered in this paper, this four-fold ambiguity is maximal.}
\label{fig:surface-sff}
\end{center}
\end{figure}

\section{The Shading Flow}

A smooth surface patch under diffuse Lambertian lighting and orthogonal projection yields
smooth image curves of constant brightness. The shading flow derives from these level curves of image intensity $I(x, y)$.
To construct a vector field  $V(x, y)$, we quantize these curves by taking their tangents over a predetermined image coordinate grid.  We call this vector field of isophote tangents the \emph{shading flow field}.
 In the limit, as the spacing of the grid points goes to zero, integral curves of the shading flow are precisely the intensity level curves.  Our goal is to use this 2D vector field to restrict the family of (surface, light source direction) pairs that could have resulted in the image (Fig. \ref{fig:workflow}).  In addition, we will use the complementary vector field of brightness gradients.  Note that, in the limit, these two vector fields of isophote tangents and brightness gradients together encapsulate the same information as the pixel values of the image.

Due to space limitations, we cannot analyze the shading flow here.  For computational work on regularizing and calculating the shading flow, see \cite{zucker}.  We simply remark that (i)
it regularizes certain errors due to noise in images, and (ii) it is invariant under some transformations of the albedo and surface. Importantly,  the shading flow field is also lower dimensional than the image: it can be described as a collection of angles at subsampled image positions.
%

Recoving shape from the shading flow will always be ill-posed; see the
ambiguity in Fig.~\ref{fig:surface-sff}.  However, at certain points on the surface, the ambiguities collapse in dimension.  For example, along the boundary of a smooth object, the view vector lies in the tangent plane \cite{Koenderink84}.   Along a suggestive contour \cite{decarlo03}, the dot product of the normal vector and the view vector is at a local minimum in the direction of the viewer.  At a highlight, the dot product of the light source direction and normal vector are at a local maximum.  All these types of points are identifiable in the image and provide additional geometric information that reduces ambiguity locally.  This leads to a plan for reconstruction:  First, parametrize the shading ambiguity in the general shading case.   Then, locate the points where the shading ambiguity vanishes (or reduces greatly) and solve for the local surface shape at those points.  Finally, calculate the more complex regions via a compatibility technique \cite{zucker} or interpolation.  In this paper, we focus on just the first step.  Papers corresponding to the second and third steps will be forthcoming.

\begin{figure}[t]
\begin{center}
\includegraphics[width= 1 \linewidth]{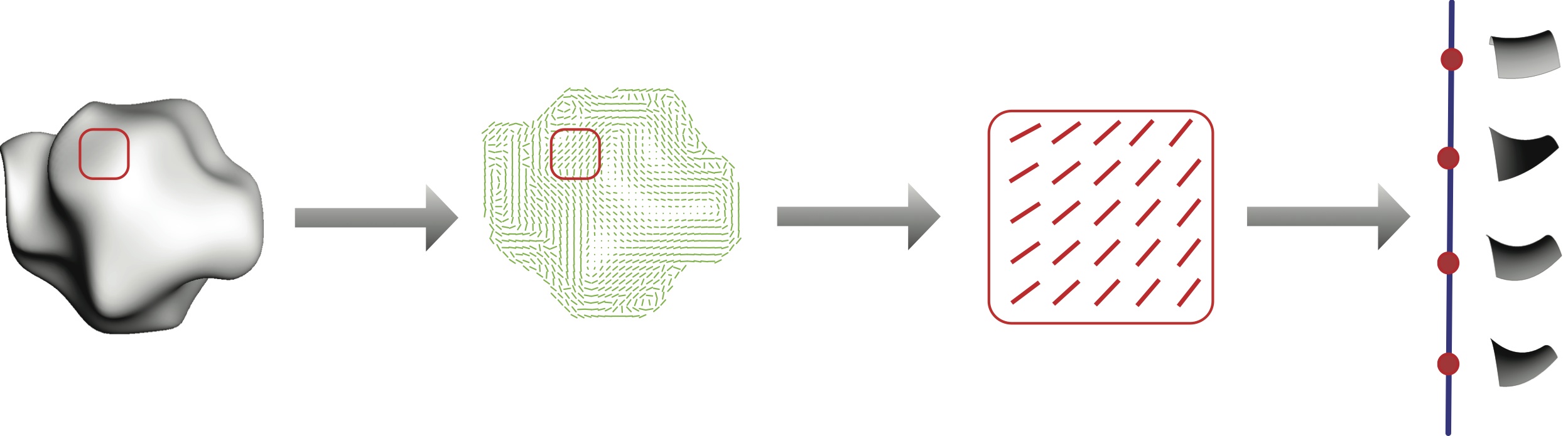}
\caption{This figure represents the workflow going from image to shading flow to a set of local surfaces.  Each surface along the \emph{fiber} needs a particular light source position to correspond to the given shading flow.  In the Application Section, we solve this particular problem for all 2nd order surfaces.}
\label{fig:workflow}
\end{center}
\end{figure}
In summary, we have the {\bf Problem Statement:}

\indent {\em Assume that a smooth Lambertian surface (representable as a graph of a function) with locally constant albedo is lit from any finite number of unknown directions.  Assume the image is captured through orthogonal projection.  Given the shading flow and brightness gradient vector fields, recover the entire set of surfaces consistent with the image information.}

\section{Analysis}
\subsection{Outline}
 Our goal is to translate the problem into the local tangent plane and then use the machinery of covariant derivatives and parallel transport to represent image derivatives (see Fig.\ref{fig:im_flow}) as a function of the surface vector flows.  A similar use of this machinery was applied for shape from texture in \cite{Garding95}.


\vspace{5mm}
\begin{enumerate}
  \item We write the brightness gradient and isophote as tangent plane conditions between the projected light source and shape operator.
  \item We take the covariant derivative of the projected light source and show it is independent of the direction of the light source.
  \item We take the covariant derivative of the isophote condition and separate into the differentiation on the projected light source and the differentiation on the shape operator.  
  \item  We take the covariant derivative of the shape operator applied to the isophote vector.  This requires several steps.
    \begin{enumerate}
    \item Expansion of the $dN$ operator in terms of the Hessian $H$.
    \item Covariant differentiation of $H$.
    \item Parallel transport of $H$.
    \item Covariant differentiation of $dN$ from covariant differentiation of $H$.
  \end{enumerate}
  \item Substitution and algebra
\end{enumerate}

\subsection{Notation}

\begin{figure}[t]
\begin{center}
\includegraphics[width= 0.6\linewidth]{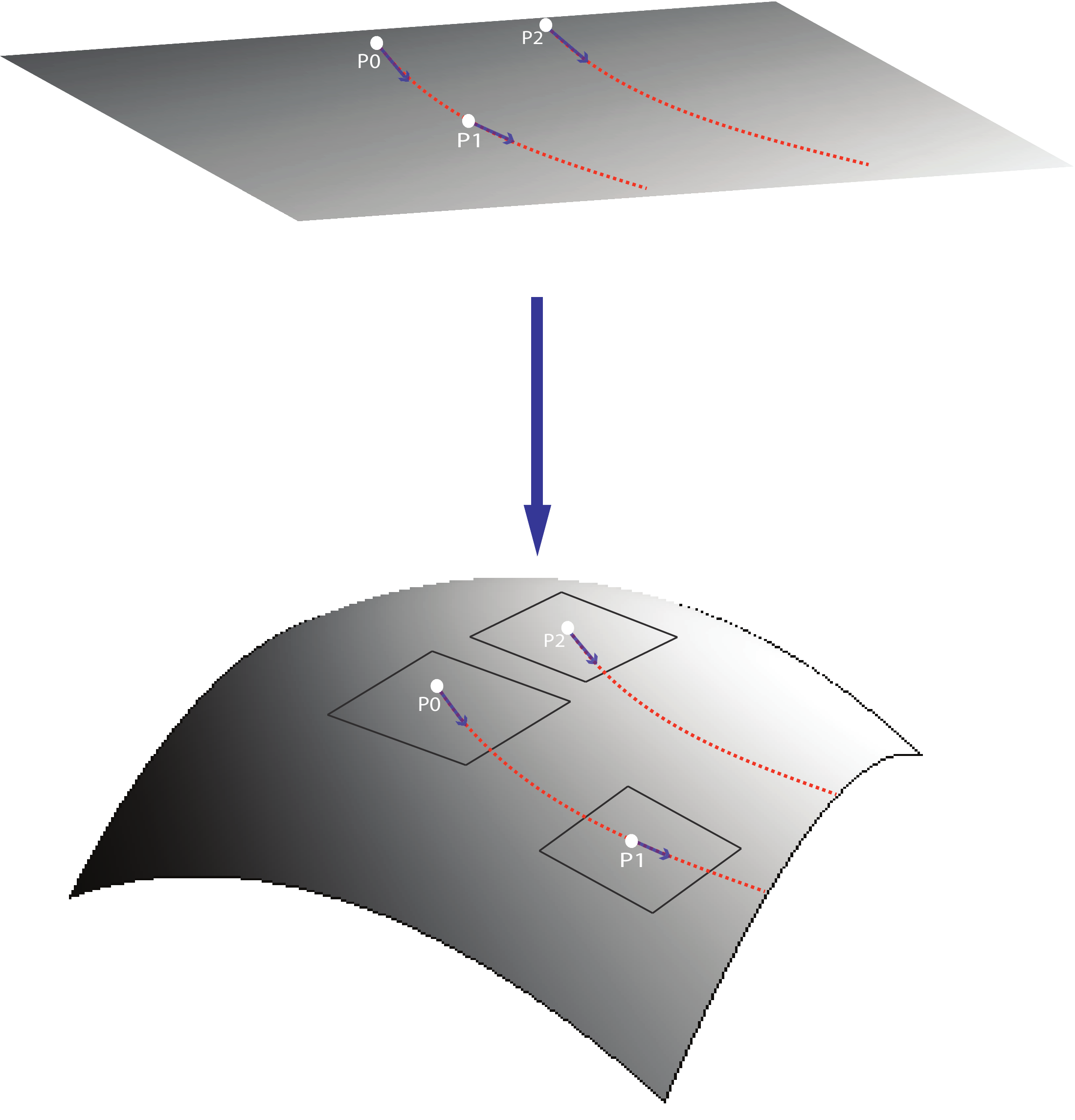}
\caption{Geometric setting for the shape-from-shading-flow problem. A surface (bottom) projects orthographically onto the image (top). Isophotes on the surface project to curves in the image. The shading flow is the tangent map to these image contours. When pulled back onto the surface, a tangent vector from the shading flow corresponds to a vector in the tangent plane to the surface. The natural operation on images is to ``walk along'' the shading flow, either along or across isophotes. We calculate the analagous operations on the surface to understand the conditions on surface curvature that result from changes in the shading flow.}
\label{fig:im_flow}
\end{center}
\end{figure}

The Lambertian lighting model is defined by:

$$I(x, y) = \rho \vec{L} \cdot N(x, y)$$

We derive local shading equations here.  Consider a small image patch under Lambertian lighting from an unknown light source.
This image patch corresponds to a local surface patch, which by Taylor's theorem we will represent as $S = \{ x, y, f(x, y) \}$ with $f(x, y) = c_1 x + c_2 y + c_3 x^2 + c_4 x y + c_5 y^2 + c_6 x^3 + c_7 x^2 y + c_8 x y^2 + c_9 y^3$.

Our goal is to understand the derivatives of intensity in terms of the coefficients $\{ c_i \}$.  It is essential that the order of the Taylor polynomial must be 3 since we shall consider second derivatives of image intensity and intensity is already dependent (via Lambertian lighting) on the first order derivatives of the surface.  Other analyses of SFS only consider 2nd order Taylor approximations \cite{pentland84}.

For reference, we define our complete notation in the Table \ref{tab:notTable}.  The symbols will be introduced throughout the analysis.  $V(x, y)$ is the shading flow field.  We normalize $V(x, y)$ to be of unit length in the image, although the corresponding surface tangent vectors have unknown length.   We denote unit length vectors in the image plane with a vector superscript, such as $\vec{v}$. The corresponding vectors on the surface tangent plane are defined by the image of  $\vec{v}$ under the  map composition of the differential $df: \mathbb{R}^2 \rightarrow \mathbb{R}^3$ and the tangent plane basis change $T: \mathbb{R}^3 \rightarrow T_p (S)$.  We will use the hat superscript to denote these surface tangent vectors, e.g. $\hat{v}$. 

Thus, 
$$\hat{v} =  T \circ df ( \vec{v})$$
\begin{table}[h]
\footnotesize
\caption{Notation Table} 
\centering 
\begin{tabular}{l c} 
\hline 
$p$ & a chosen point $(x_0, y_0)$ \\
$I_{p} (x, y)$ & an image patch centered at $p$ \\
$\nabla I (x, y)$ &  the brightness gradient \\
$S_{p} (x, y)$ & the corresponding (unknown) surface patch \\
$f(x, y)$ & the Taylor approximation  at $p$ of $S$ \\
$\{c_i \}$ & the coefficients of the Taylor approximation $f(x, y)$ \\
$T_p (S)$ & the tangent plane of $S$ at $p$ \\
$\vec{L}$ & the light source direction \\
$\vec{l_t} (p)$ & the projection of the $L$ onto the tangent plane \\
$\vec{e_i}$ & unit length standard basis vector in direction of coordinate axis $i$ \\
$N(x, y)$ & the unit normal vector field of $S$ \\
$V(x, y)$ & the vector field of isophote directions at each point (x, y) \\
$\vec{v} \in T_p (S)$ & the image unit length tangent vector in the direction of the isophote at $p$\\
$\vec{u} \in T_p (S)$ & the image unit length tangent vector in the direction of the brightness gradient at $p$\\
$\hat{w} \in T_p (S)$ & the tangent vector in direction $\vec{w}$ of unit length in the image, \\
& expressed in the surface tangent basis \\
$\vec{u} [ V ]$ & the directional derivative of the vector field $V$ in the direction $\vec{u}$ \\ 
$\nabla_{\vec{u} } V$ & the covariant derivative of the vector field $V$ in the direction $\vec{u}$ \\ 
$G$ & the first fundamental form (also called the metric tensor) \\
$II$ & the second fundamental form \\
$H$ & the Hessian \\
$dN$ & the differential Gauss Map, also called the Shape Operator \\
\hline 
\end{tabular}
\label{tab:notTable}
\end{table}
\begin{figure}[t]
\begin{center}
\includegraphics[trim = 100 0 0 0,width= 1 \linewidth]{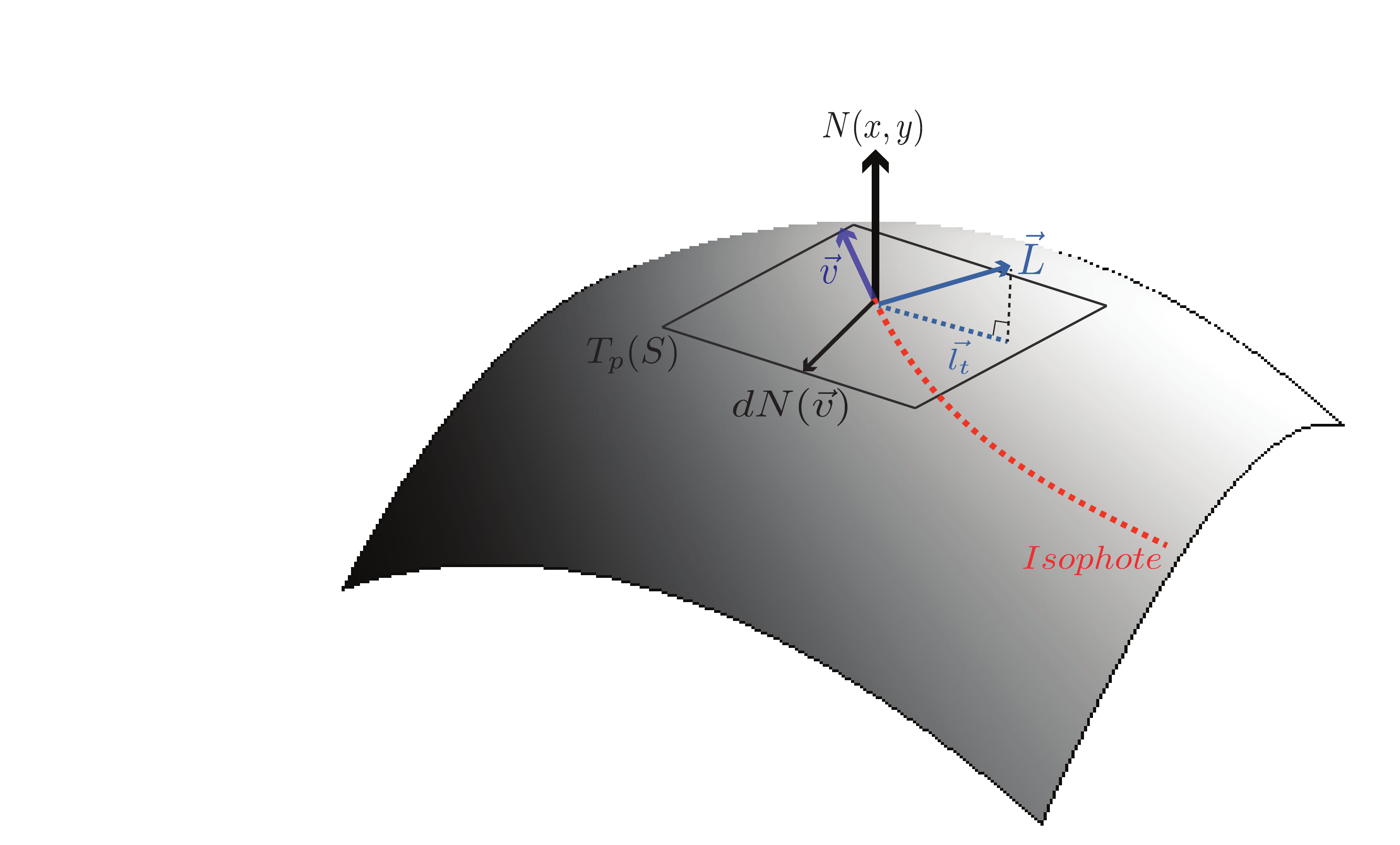}
\caption{A diagram explaining our defined surface properties.}
\end{center}
\end{figure}
\subsection{Brightness Gradient} 
\label{sec:brgr}

We derive the equations for the brightness gradient as a function of the light source and the second fundamental form.  Similar derivations (with different notation) appear in $\cite{Koenderink90}$.

The brightness gradient $\nabla I$ can be defined as a linear 1-form having as input unit length image vectors $\vec{w}$ and having as output a real number.  The output is the change in brightness along a step on the surface using $\hat{w}$. We write:

\begin{subequations}
\begin{align}
\nabla I \cdot \vec{w} & = \hat{w} \left[ \langle \vec{L}, \vec{N} \rangle \right] \\
& =  \langle \left( \nabla_{\hat{w}} \vec{L} \right) , \vec{N}  \rangle +  \langle \vec{L} , \left( \nabla_{\hat{w}} \vec{N} \right) \rangle 
\label{eqn:temp1}\\
& = 0 + \langle \vec{L}, dN(\hat{w}) \rangle\\
& = \langle \vec{l_t}, dN(\vec{w}) \rangle \\
& =  \vec{l}^T II \hat{w}
\end{align}
\end{subequations}

\noindent where the first term in equation (\ref{eqn:temp1}) is zero because the light source is fixed.  

\begin{proposition}
The brightness gradient $\nabla I$ can be expressed as the vector $\vec{l}_t^T II$.
\end{proposition}

Along an isophote surface curve $\alpha(t)$, the brightness is constant.  Writing  $\vec{v} = \alpha ' (0)$, we have $ \langle \vec{l_t}, dN(\vec{v}) \rangle = \vec{l}^T II \hat{v} = \nabla I \cdot \vec{v} = 0$.

Thus, we conclude:
\begin{proposition}
Each isophote tangent vector $\vec{v}$ on $S$ is a function of the normal curvatures and light source and is defined by the equation 	
$$ \langle \vec{l_t}, dN(\vec{v}) \rangle = 0.$$
\end{proposition}
In addition, we calculate each component of the brightness gradient via dot product with $\vec{e_i}$.
\begin{align}
I_x & =  \langle \vec{l_t}, dN(\vec{e_1}) \rangle \\
I_y & =   \langle \vec{l_t}, dN(\vec{e_2}) \rangle 
\end{align}

\subsection{Covariant Derivative of Projected Light Source}

One of the major advantages of our approach is that we do not need to assume a known light source direction.  In fact, using the covariant derivative described below, we can calculate the change in the projected light source vector without knowing where it is!  

\vspace{5mm}
\begin{remark}
Here, we briefly remark on the use of covariant derivatives for surfaces in $\mathbb{R}^2$.   We consider ``movements" in the image plane and sync them with ``movements" through the tangent plane bundle on the surface.  The problem is that the image plane vectors are on a flat surface, whereas the vectors on the surface tangent planes ``live" in different tangent spaces:  the surface tangent planes are all different orientations of $\mathbb{R}^2$ in $\mathbb{R}^3$.  Thus, to calculate derivatives via limits of differences, we need to ``parallel transport" nearby vectors to a common tangent plane.  The covariant derivative achieves this.  For our purposes, we think of the covariant derivative in two ways.  The first definition, which we use in this section, is the expression as the composition of a derivative operator in $\mathbb{R}^3$ and a projection operator onto a tangent plane.  This is an \emph{extrinsic definition} -- it is a definition that requires use of the ambient space.  The second definition, which we will use in Section \ref{sec:transport}, will be in terms of  parallel transport.
\end{remark}
\vspace{5mm}

We exploit the structure in $\vec{l}_t$:  it is the result of a projection from a fixed vector $L$ down into the tangent plane $T_p (S)$.  Thus, the change in $\vec{l}_t$ just results from changes in the tangent plane, which is dependent only on the surface curvatures and not on $L$.  Importantly, we avoid having to represent $L$ in our calculations by only considering its projected changes.  We now show this rigorously.
\begin{lemma}  
The covariant derivative of the projected light source is only dependent on the position of the light source through the observed intensity. Thus, 
$$\nabla_{\vec{u}} \vec{l_t} = - (\vec{L} \cdot \vec{N}) dN(\vec{u}).$$
\end{lemma}
\begin{proof}
Let $\Pi_{p_0}$ be the projection operator taking a vector in 3-space onto the tangent plane of $S$ at $p_0$.  Recall that the covariant derivative of a tangent vector can be expressed as the composition of a derivative operator and $\Pi$.

\begin{subequations}
\begin{align}
\nabla_{\vec{u}} \vec{l_t} & = \Pi_{p_0} \left( \frac{d \vec{l_t}}{dt} \right) \\
		   & = \Pi_{p_0} \left( \frac{d}{dt} (\vec{L} - (\vec{L} \cdot \vec{N}) \vec{N}) \right) \\
		   & = \Pi_{p_0} \left( \frac{d\vec{L}}{dt} - \frac{d}{dt} \left[ (\vec{L} \cdot \vec{N}) \vec{N} \right] \right) \\
		   & =\Pi_{p_0} \left(  0 - \frac{d}{dt} \left[ \vec{L} \cdot \vec{N} \right] \vec{N} - (\vec{L} \cdot \vec{N}) \frac{d\vec{N}}{dt} \right) \\
		   & =  \Pi_{p_0} \left( - \left[ \frac{d\vec{L}}{dt} \cdot \vec{N} +  \vec{L} \cdot \frac{d\vec{N}}{dt} \right] \vec{N} - (\vec{L} \cdot \vec{N}) dN(\vec{u}) \right) \\
		   & = -(\vec{L} \cdot \vec{N}) dN(\vec{u}) \label{eq:ls_change}
\end{align}
\end{subequations}
\end{proof}  

The fact that this change in the projected light source only depends on surface properties allows us to remove the light source dependence from the second derivatives of intensity $\{I_{vv}, I_{uv}, I_{uu} \}$.

\subsection{Covariant Derivative of the Isophote Condition}
\begin{figure}[t]
\begin{center}
\includegraphics[trim = 0 10mm 10mm 28mm, clip=true, width= 3.1in]{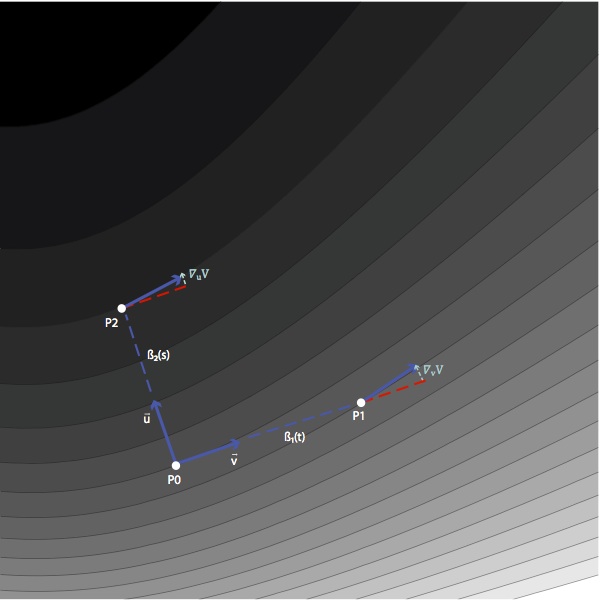}
\caption{A diagram explaining our use of the shading flow field.  As we move on $\beta_1(t)$ in the direction of the isophote 
$\vec{v}$ from $P0$ to $P1$, the flow field $V(x, y)$ changes by $\nabla_{\vec{v}} V$.   Similarly, we may move in direction
$\vec{u}$ along $\beta_2(s)$, which is perpendicular (in the image) to the isophote.  Then, our flow field changes by 
$\nabla_{\vec{u}} V$.  In Proposition 2, we relate these changes in closed form to the curvatures of the surface and the light source direction.}
\label{fig:sff-steps}
\end{center}
\end{figure}

We now use the changes in the brightness gradient and the isophote directions to restrict our surface parameters.  Let $\vec{v}$ be the unit length image vector in the direction of the isophote at an arbitrary point $p$.  Let $\vec{u}$ be the unit length image vector in the direction of the brightness gradient at $p$.  In the image, $\vec{v} \perp \vec{u}$ but the projected vectors $\hat{v}$ and $\hat{u}$ may not be orthogonal on the tangent plane at $p$.

In fact, considering these particular changes in $\vec{u}$ and $\vec{v}$ is equivalent to choosing a basis.  This will result in solving for equations of the three second derivatives $\{I_{vv}, I_{uv}, I_{uu} \}$, although we could have considered the changes in $\{I_x, I_y \}$ and instead solved for $\{I_{xx}, I_{xy}, I_{yy} \}$.  However, the equations simplify when choosing the basis defined by the isophote and brightness gradient.

To emphasize the conceptual picture, we will derive the $I_{vv}$ equation here and save the equations in the general case for the appendix.  

We start by first calculating $I_v$ and then taking the directional derivative of $I_v$ in the direction $\vec{v}$.  From Section \ref{sec:brgr}, we can write:

\begin{align}
0 = I_v = \nabla I \cdot \vec{v} =   \langle \vec{l_t}, dN(\vec{v}) \rangle 
\end{align}

\noindent Applying the directional derivative with respect to $\vec{v}$ on both sides and using the result from Equation \ref{eq:ls_change}:

\begin{subequations}
\begin{align}
0 & = v \left[  \langle \vec{l_t}, dN(\vec{v}) \rangle \right] \\
& = \langle \left( \nabla_{\vec{v}} \vec{l_t} \right), dN(\vec{v}) \rangle + \langle \vec{l_t},  \nabla_{\vec{v}} dN(\vec{v}) \rangle \\
& = \langle -(\vec{L} \cdot \vec{N}) dN(\vec{v}), dN(\vec{v}) \rangle + \langle \vec{l_t},  \nabla_{\vec{v}} dN(\vec{v}) \rangle \\
& = - I \langle dN(\vec{v}), dN(\vec{v}) \rangle + \langle \vec{l_t},  \nabla_{\vec{v}}( dN(\vec{v})) \rangle \label{eq:pre}
\end{align}
\end{subequations}

We now unpack  $\langle \vec{l_t},  \nabla_{\vec{v}} dN(\vec{v}) \rangle$ which requires a technical computation using parallel transport and tensor algebra.
 
\subsection{Covariant Derivative of the Shape Operator $dN$}

We expand the RHS term of (\ref{eq:pre}).  This will also expand into several terms; we will analyze each separately.  Although $dN(v)$ is an unknown vector, we want to understand it's covariant derivative $\nabla_{\vec{u}} (dN(\vec{v}))$ as a function of surface changes $\nabla_{\vec{u}}(dN)$ and isophote changes $\nabla_{\vec{u}} \vec{v}$.  We apply the chain rule to get:

\begin{align}
 \nabla_{\vec{v}} (dN(\vec{v})) = \left( \nabla_{\vec{v}}dN \right) (\vec{v}) + dN(\nabla_{\vec{v}} \vec{v})
 \label{eq:mid}
\end{align}

We now focus on the first term.

\subsubsection{Expansion of $dN$ in terms of the Hessian $H$}
Note that $dN$ is $(1, 1)$ tensor and thus we need to be careful when taking its covariant derivative.  Recall that the matrix representation of $dN$ is $I^{-1} II$ and that raising and lowering the tensor characteristic commutes with covariant differentiation.

\begin{subequations}
\begin{align}
\left( \nabla_{\vec{v}}dN \right) (\vec{v}) & = \left( \nabla_{\vec{v}}G^{-1} II \right) (\vec{v}) \\
& = \left(G^{-1} \nabla_{\vec{v}}II \right) (\vec{v}) \label{eq:16}
\end{align}
\end{subequations}

Write each of the normal components as $n_i$ so $\vec{N} = \{n_1, n_2, n_3 \}$.  Note that due to our Monge patch representation of $S(x, y)$, $\{\vec{f_{xx}}, \vec{f_{xy}}, \vec{f_{yy}}\}$ are nonzero only in their third component, e.g. $\vec{f_{xx}} = \{0, 0, g_{xx}\}$.
Recall the definition of the second fundamental form $II$:

\begin{subequations}
\begin{align}
II & = \begin{bmatrix} \vec{N} \cdot \vec{f_{xx}} & \vec{N} \cdot \vec{f_{xy}} \\  \vec{N} \cdot \vec{f_{xy}} & \vec{N} \cdot \vec{f_{yy}} \end{bmatrix}\\
& = \begin{bmatrix} n_3 g_{xx} & n_3 g_{xy} \\  n_3  g_{xy} &  n_3 g_{yy} \end{bmatrix}\\
& = n_3 \begin{bmatrix} g_{xx} & g_{xy} \\  g_{xy} & g_{yy} \end{bmatrix}
\end{align}
\end{subequations}

\noindent For notational convenience, we use the Hessian $H =  \begin{bmatrix} g_{xx} & g_{xy} \\  g_{xy} & g_{yy} \end{bmatrix}$.  Substituting into \ref{eq:16} and using the appropriate product rule:

\begin{subequations}
\begin{align}
\label{eq:nonzero_form_term}
\left( \nabla_{\vec{v}}dN \right) (\vec{v})  & = G^{-1}  \nabla_{\vec{v}} \left( n_3 H \right) (v)\\
\label{eq:mid_dN}
& = n_3 \, G^{-1}   \nabla_{\vec{v}}  \left(H \right) (v) + \vec{v} [ n_3] \, G^{-1} H  (v)
\end{align}
\end{subequations}

\noindent The second term of the above equation will be exactly zero (since $\vec{l_t} II \vec{v} = 0$) after the dot product with $\vec{l_t}$ in \ref{eq:pre}.  Thus, we only need calculate the first term and particularly the covariant derivative of $H$.

\subsubsection{Covariant Differentation of  $H$}
To covariant differentiate $H$, we note that $H$ is a $(0, 2)$ tensor and so we can expand it as a sum of tensor products of 1-forms.  We follow the notation in \cite{Dodson91}.  Write $H^1, H^2$ as the two rows of $H$ and $E_1, E_2$ as the standard basis 1-forms.  In a tensor representation, $H^1$ and $H^2$ are also both 1-form fields (covariant tensors).

\begin{subequations}
 \begin{align}
 H & = \begin{bmatrix} g_{xx} & g_{xy} \\ g_{xy} & g_{yy} \\ \end{bmatrix} \\
 & = \begin{bmatrix} g_{xx} & g_{xy} \\ \end{bmatrix} \otimes \begin{bmatrix} 1 & 0 \\ \end{bmatrix} + \begin{bmatrix} g_{xy} & g_{yy} \\ \end{bmatrix} \otimes \begin{bmatrix} 0 & 1 \\ \end{bmatrix}\\
 & = H^1 \otimes \vec{E}_1 + H^2 \otimes \vec{E}_2
 \end{align}
 \end{subequations}

\noindent To covariant differentiate $H$, we apply a product rule for tensor products:

\begin{align}
\label{eq:t_sum}
 \nabla_{u} H & = H^1 \otimes \nabla_v E_1 + H^2 \otimes \nabla_v E_2 \\
 			& + \nabla_v H^1 \otimes E_1 + \nabla_v H^2 \otimes E_2 \notag 
\end{align}
 
 \noindent Note that each of these 4 terms is also a $(0, 2)$ tensor and thus each term requires as input two vectors.  
Without loss of generality, we calculate one of the individual terms $ \nabla_v H^1 \otimes E_1$.  The rest are analogous.  
The covariant derivative of a covariant tensor requires actions on the tensor inputs, so we introduce dummy vectors $w_1, w_2$ to use in our expression.  

\subsubsection{Parallel Transport of $H$}
\label{sec:transport}

\vspace{5mm}

\begin{remark}
  As mentioned in Remark 1, we recall the second, equivalent definition of covariant differentiation here.  We define it intrinsically, that is, independent of the ambient space $\mathbb{R}^3$.  We will not go into the derivations regarding \emph{connections} or Christoffel symbols, which can be found in \cite{Dodson91} and \cite{docarmo}.  We just summarize that \emph{parallel transport} is a way to ``equate" nearby vectors in nearby tangent planes along a curve $\beta(s)$.  Using notation as in \cite{Dodson91}, we will write the parallel transport  in the forward direction of the vector field $\vec{w}(\beta(s))$ as $\tau_{s}^{\rightarrow} (\vec{w}(\beta(s)))$.  Conversely, the parallel transports backwards along the curve is written $\tau_{s}^{\leftarrow} (\vec{w}(\beta(s)))$.  Then, the covariant derivative can be defined intrinsically as:

$$\nabla_{\beta ' (0)} \vec{w} = \lim_{s \rightarrow 0} \left( \frac{\tau_{s}^{\leftarrow}(\vec{w}(\beta(s)) - \vec{w}(\beta(0))}{s} \right)$$

Thus, covariant differentiation resolves the tangent plane orientation problem by first transporting the vector $\vec{w}(\beta(s)) \in T_{\beta(s)} (S)$ back to a  ``parallel" vector in $T_{\beta(0)} (S)$  before doing the standard derivative subtraction.
\end{remark}
\vspace{5mm}

Now, due to the duality between 1-forms and vectors, when we apply covariant differentiation to a 1-form, we parallel transport the vector it acts on forwards.  (This is opposite to a covariant derivative of  a vector which is parallel transported backwards in the derivative.)  Define $\beta(s)$ as a curve passing through $P$ with velocity $\vec{v}$.  Note that the 1-form $H^1$ and its input vectors $w_1$ are defined along $\beta(s)$ and may be indexed at different positions.  We will denote the position of $H^1$ using a subscript, such as $H^1_{\beta(s)}$.  Using the definition of covariant derivative for covariant tensors in \cite{Dodson91}:

\begin{subequations}
\begin{align}
(\nabla_{\vec{v}} H^1 \otimes E_1)(\vec{w_1}, \vec{w_2}) & = E_1(\vec{w_2}) \smul \left( \lim_{s \rightarrow 0} \frac{ H^1_{\beta(s)} (\tau_{s}^{\rightarrow} (\vec{w_1} (\beta(0)))) - H^1_{\beta(0)} (\vec{w_1}(\beta(0))) }{s}\right) \label{eq:lim_a} \\
 & = E_1(\vec{w_2}) \smul \left( H^1_{\beta(0)} \left( \lim_{s \rightarrow 0} \frac{(\tau_{s}^{\rightarrow} (\vec{w_1} (\beta(0)))) - (\vec{w_1}(\beta(0))) }{s}\right)\right) \label{eq:lim_b}\\
 & \quad  - E_1(\vec{w_2}) \smul \left( \vec{v} [ H^1] \vec{w_1}  \right) \\
& = E_1(\vec{w_2}) \smul H^1\left( \vec{T_{w_1}} \right)  - E_1(\vec{w_2}) \smul \left( \vec{v} [ H^1] \vec{w_1}  \right)
\end{align}
\end{subequations}

\noindent The dot represents multiplication and for simplicity of notation, we assigned the change in parallel transport of an arbitrary vector $w_1$ to be $\vec{T_{w_1}}$: 
$$\left( \lim_{s \rightarrow 0} \frac{(\tau_{s}^{\rightarrow} (\vec{w_1} (\beta(0)))) - (\vec{w_1}(\beta(0))) }{s}\right) = \vec{T_{w_1}}$$

\noindent $\vec{T_{w_1}}$ can be written using the \emph{Christoffel symbols} but this will lead to unnecessary notation as the terms involving $\vec{T_{w_1}}$  will eventually cancel.
To go from equations \ref{eq:lim_a} to \ref{eq:lim_b}, we used a substitution: 

$$ - \vec{v} [ H^1] (\vec{w_1}) =  \lim_{s \rightarrow 0} \left( \frac{- H^1_{\beta(s)} (\tau_{s}^{\rightarrow} (\vec{w_1} (\beta(0)))) + H^1_{\beta(0)} (\tau_{s}^{\rightarrow} (\vec{w_1}(\beta(0))))}{s} \right) $$ 

We now repeat the process for the other three terms in \ref{eq:t_sum} and compile the four terms into the original matrix representation to get:


\begin{subequations}
\begin{align}
\label{eq:CD_A}
 \left( \nabla_{v} H \right) (\vec{w_1}, \vec{w_2}) & = - \vec{T_{w_2}} H \vec{w_1} \\
 & {\quad} - \vec{w_2}^T H \vec{T_{w_1}} 
  + (\vec{w_2} \vec{v} [ H] \vec{w_1}) 
\end{align}
\end{subequations}

\subsection{Covariant differentiation of $dN$ from covariant differentiation of $H$}
We have calculated the covariant derivative of the matrix $H$ for arbitrary tangent vectors.  Now, we substitute that in to finish the expansion of the initial equation \ref{eq:pre}.
We apply the above equation  \ref{eq:CD_A} to the equation \ref{eq:mid_dN} to get:

\begin{subequations}
\begin{align}
\langle \vec{l_t}, \left( \nabla_{\vec{v}}dN \right) (\vec{v}) \rangle  & = \Big\langle \vec{l_t}, n_3 G^{-1}   \nabla_{\vec{v}}  \left( \begin{bmatrix} g_{xx} & g_{xy} \\  g_{xy} & g_{yy} \end{bmatrix} \right) (v) + \vec{v} [ n_3] G^{-1} \begin{bmatrix} g_{xx} & g_{xy} \\  g_{xy} & g_{yy} \end{bmatrix}  (v) \Big\rangle \\
& = n_3 (\nabla_v H)(\vec{v}, \vec{l_t}) + \vec{v}[n_3] H (\vec{v}, \vec{l_t}) \\
& = n_3 (\nabla_v H)(\vec{v}, \vec{l_t}) \\
\label{eq:CS_v}
& = - \vec{T_{l_t}} II \vec{v} - \vec{l_t}^T II \vec{T_{v}}  + n_3  \vec{l_t}^T (\vec{v} [H ] \vec{v}) \\ 
& =  - \vec{l_t}^T II \vec{T_{v}}  + n_3  \vec{l_t}^T  (\vec{v} [H ] \vec{v})
\end{align}
\end{subequations}

\noindent where the first term of \ref{eq:CS_v} is proportional to $\vec{l_t} II v$, which is 0 by Proposition 3.2. Now, $n_3 \vec{l_t}^T$ = $n_3 \vec{l_t}^T H H^{-1} = \vec{l_t}^T II H^{-1} = (\nabla I) \cdot H^{-1}$, so 

$$\langle \vec{l_t}, \left( \nabla_{\vec{v}}dN \right) (\vec{v}) \rangle   = - \vec{l_t}^T II \vec{T_{v}} +  (\nabla I) \cdot H^{-1} \vec{v} [H ] \vec{v}$$

\subsection{Putting it all together}
Substituting the above equation into \ref{eq:mid}:

\begin{subequations}
\begin{align}
 \nabla_{\vec{v}} (dN(\vec{v})) & = \langle \vec{l_t}, \left( \nabla_{\vec{v}}dN \right) (\vec{v}) + dN(\nabla_{\vec{v}} \vec{v})
\rangle \\
& = - \vec{l_t}^T II \vec{T_{v}} -  (\nabla I) \cdot H^{-1} \vec{v} [H ] \vec{v} + \langle l_t, dN(\nabla_{\vec{v}} \vec{v}) \rangle  \\
& = -  \vec{l_t}^T II \vec{T_{v}} -  (\nabla I) \cdot H^{-1} \vec{v} [H ] \vec{v} + \langle l_t, dN(v'(s)  - T_{v})) \rangle \\
& =   -  (\nabla I) \cdot H^{-1} \vec{v} [H ] \vec{v}+ (\vec{l_t} II) \cdot v'(s) \\
& =   - (\nabla I) \cdot H^{-1} \vec{v} [H ] \vec{v} -  \nabla I \cdot v'(s)\\
& =   - (\nabla I) \cdot H^{-1} \vec{v} [H ] \vec{v} - I_{vv}
\label{eq:end_Ivv}
\end{align}
\end{subequations}

\noindent where we have used the fact that a covariant derivative is the sum of the changes  due to parallel transport and the coordinate changes  $v'(s)$ along the curve $\beta(s)$. 

Plugging into equation \ref{eq:pre} and rearranging, we get:

$$I_{vv} =  - (\vec{L} \cdot \vec{N}) || dN(v) ||^2 +  \nabla I H^{-1} (\vec{v} [H ] \vec{v})$$

\section{Shading Equations}

We have now computed the covariant derivative of the vector $\vec{v}$ in the direction $\vec{v}$. 
For an arbitrary point $p$, let $\vec{u}$ be the image vector in the direction of the brightness gradient.  Define $\vec{u}$ to be of unit length.  Then, we can repeat this calculation for the covariant differentiation of $\vec{v}$ in the direction $\vec{u}$.  In addition, we can calculate the covariant derivative of the vector $\vec{u}$ in the direction $\vec{u}$.  (Both of these proofs are similar to the one above and are left to the Appendix \ref{App:AppendixA}.)  This gives us a total of three equations equating the second order intensity information (as represented in vector derivative form) directly to surface properties.  

\vspace{5mm}

\begin{theorem}
For any point $p$ in the image plane, let $\{ \vec{u} , \vec{v} \}$ be the local image basis defined by the brightness gradient and isophote.  Let $I$ be the intensity, $\nabla I$ be the brightness gradient, $f(x, y)$ be the height function, $H$ be the Hessian, and $dN$ be the shape operator.  Then, the following equations hold regardless of the light source direction:

\begin{equation}
I_{vv} =  - I || dN(\vec{v}) ||^2 +  (\nabla I) \cdot H^{-1} (\vec{v} [H ] \vec{v})
\label{eq:shd-theorem1}
\end{equation}

\begin{equation}
I_{uu} =   - I || dN(\vec{u}) ||^2 - 2 \frac{|| \nabla I ||}{\sqrt{1 + || \nabla f ||^2}} \langle \nabla f, dN(\vec{u}) \rangle +  (\nabla I) \cdot  H^{-1} \cdot (\vec{u} [H ] \vec{u})
\label{eq:shd-theorem2}
\end{equation}
\end{theorem}

\begin{equation}
I_{uv} =  - I \langle dN(\vec{v}), dN(\vec{u}) \rangle - \frac{|| \nabla I ||}{\sqrt{1 + || \nabla f ||^2}} \langle \nabla f, dN(\vec{v}) \rangle +   (\nabla I) \cdot H^{-1} \cdot (\vec{u} [H ] \vec{v})
\label{eq:shd-theorem3}
\end{equation}

These equations are novel; we call them \emph{the 2nd-order shading equations}.  Note that there is no dependence on the light source; thus, these equations directly restrict the derivatives of our local surface patch.  We have included in the appendix the expanded versions of these equations in terms of those derivatives.

\section{Applications of the shading equations}

Below, we illustrate applications of these second order shading equations.  We consider four applications, arranged from simple to complicated :
\begin{enumerate}
\item In the 1D case, these equations can be solved directly in a partial differential equation (P.D.E.) formulation to recover the curve exactly.
\item A second order surface assumption, with a fronto-parallel tangent plane, leads to four explicit solutions.
\item A second order surface assumption with arbitrary tangent planes can be solved implicitly.
\item ``Critical" points and curves in the image have reduced shading ambiguity in the general case.
\end{enumerate}

The last case, which is much more complicated will be treated more fully in a companion paper (in prep).

\subsection{1D Shape from Shading}

\begin{figure}[t]
\begin{center}
\includegraphics[width= 0.8 \linewidth]{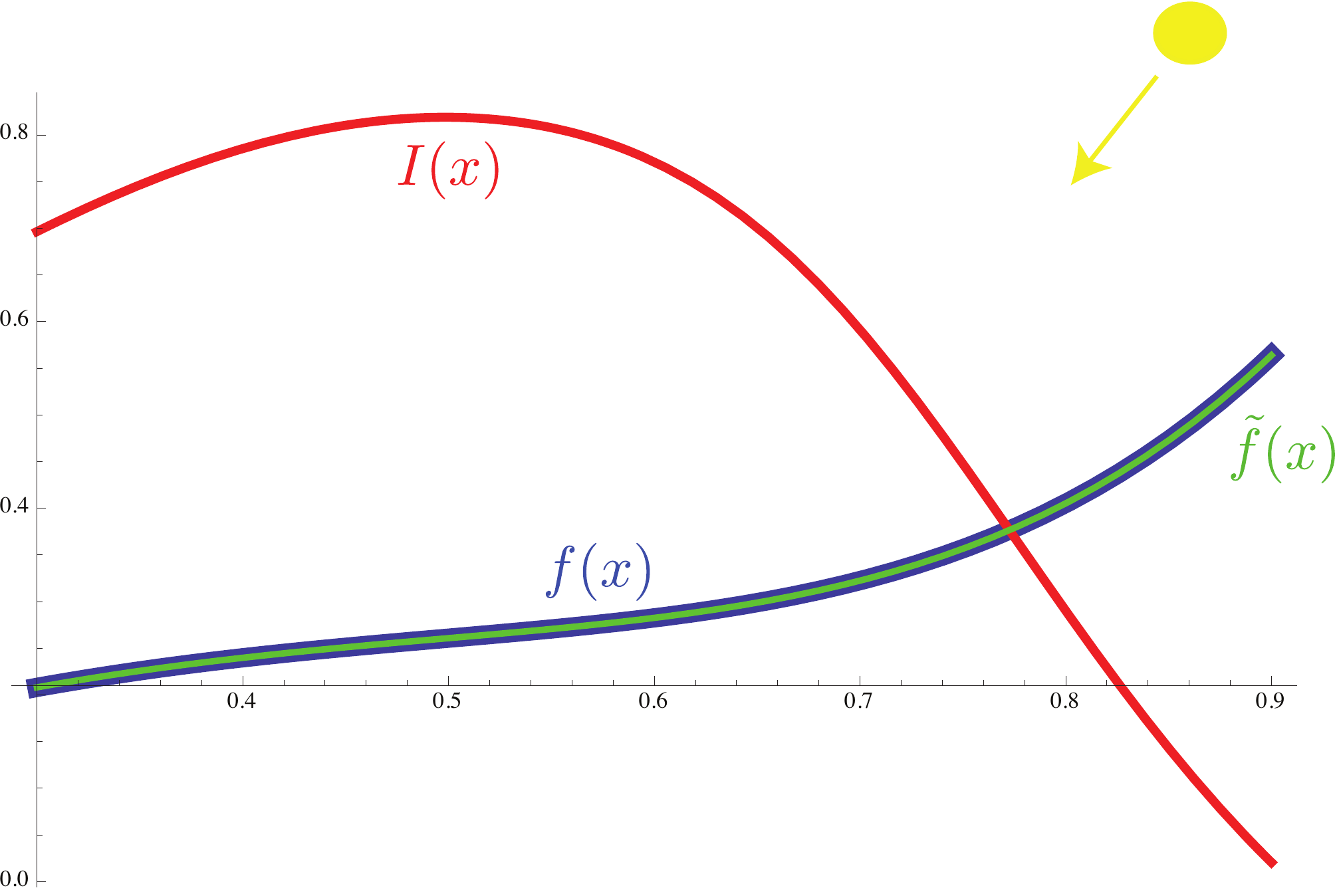}
\caption{The shape-from-shading flow problem in one dimension. The blue curve is the 1D surfaces which has the equation $\sin(x)?x^2+x^5$ . The red curve is the intensity function for the 1D setup, given a a light source of norm 1 and direction as shown. The green curve (overlayed) is the recovered surface f(x) using our P.D.E formulation 5.1 given three boundary conditions. As
you can see, the reconstruction is indistinguishable from the correct surface.}
\label{fig:1D_example}
\end{center}
\end{figure}

For simplification and to build intuition, we consider the problem of shape from shading in 1D:  Given a one dimensional intensity function $I(x)$, solve for the smooth curve $f(x)$ corresponding to $I(x)$ under Lambertian shading.  Although this problem can be solved with other means, we use it to illustrate the P.D.E approach to solving these shading equations.   In this example, we build the intensity function using an unknown light source and recover the shape exactly using our second order shading equations.   We treat our equation as a partial differential equation and solve it numerically. See Figure \ref{fig:1D_example} for the problem setup.

\subsubsection{P.D.E Formulation}
On the curve $f(x)$, the point sets of constant brightness are now single points, rather than isophote curves.  Thus, we cannot talk about $I_{vv}$ and $I_{uv}$.  However, $I_{uu}$ still makes sense, as $\vec{u}$ is now defined as the tangent $f_x$ to the curve $f(x)$.  Thus, we can apply equation  \ref{eq:shd-theorem2}.  

We need to convert the surface geometric properties in equation \ref{eq:shd-theorem2} into their simpler 1D analogs.  For example, the Hessian $H$ becomes $f_{xx}$, the directional derivative $\vec{u} [ H ] $ becomes $f_{xxx}$ and $dN(\vec{u})$ becomes the change in the normal as we move along the tangent direction:

$$dN(\vec{u}) = \frac{f_{xx}}{1+ {(f_x)}^2}$$

\noindent Putting it all together, we get the simplified shading equation in 1D:
\begin{corollary}
\begin{align}
I_{xx} = - I \frac{f_{xx}^2}{(1 + a^2)^2} - 2 I_x \frac{a f_{xx}}{(1 + a^2)} + I_x \frac{f_{xxx}}{f_{xx}}
\label{eq:1D_case}
\end{align}
\end{corollary}

Now, the functions $\{I, I_x, I_{xx} \}$ are all known from the intensity curve, so equation \ref{eq:1D_case} is a third order differential equation.  Solving this with the appropriate boundary conditions will give us our curve $f(x)$.  Note that these equations are ill-defined when $f_{xx} = 0$, so one must be careful and use approximations near these critical points.

\subsubsection{Boundary Conditions}

As equation \ref{eq:1D_case} is a third order equation, it needs three boundary conditions.   We assume known Dirichlet boundary conditions on both endpoints, with a single von Neumann condition at one of the endpoints.  With these boundary conditions, we use Mathematica's NDSolve function to get the solution curve $\tilde{f}(x)$.  See Figure 
\ref{fig:1D_example}.
Although this is a toy example, it illustrates the precision that may be available in the 2D shape from shading case if we can solve the shading equations in a P.D.E. formulation.

\subsubsection{Extension to 2D Shading}

If a surface satisfies the equations \ref{eq:shd-theorem1}, \ref{eq:shd-theorem2}, and \ref{eq:shd-theorem3} at every point in the image, then that surface will be a possible solution to shape from shading.   That is, when imaged under some set of light source position(s), it will result in an identical image.  Thus, these equations implicitly define all smooth surfaces that can satisfy a shaded image.  However, in order to solve P.D.Es, one must have boundary conditions, and it is unclear in the 2D case exactly what these boundary conditions should be.  In addition, we have no guarantee that there will be a unique solution as these are non-linear equations and thus don't satisfy the standard P.D.E uniqueness theorem.  In fact, we know that many different global surfaces can lead to the exact same image. Thus, we will consider some solutions under simplifications (equivalently, assumptions on our surface).   For completeness, we write the shading equations in the standard PDE fashion in Appendix \ref{App:AppendixB}.

\subsection{Second Order Assumption and Frontal-Parallel}
\label{sec:2nd_order_FP}

The shading equations are quite complicated, nonlinear, and of third order.  Although it may be possible to directly solve them as a partial differential equation system, we will initially look at surfaces where the equations reduce nicely.  (See also \cite{Kunsberg12}.)  Consider a second order Monge patch: $S = \{x, y, f(x, y) \}$ with $f(x, y) = a x + b y + c x^2 + d x y  + e y^2$.  Since there are no third order terms, the directional derivative of the Hessian will be 0 and the equations simplify to:

\begin{corollary}
\begin{subequations}
\begin{align}
I_{vv} & =  - (I) || dN(\vec{v}) ||^2 \label{eq:shd_2nd_1}\\ 
I_{uu} & =   - (I) || dN(\vec{u}) ||^2 -  2 \frac{|| \nabla I ||}{\sqrt{1 + || \nabla f ||^2}} \langle \nabla f, dN(\vec{u}) \rangle  \label{eq:shd_2nd_2}\\
I_{uv} & =  - (I) \langle dN(\vec{v}), dN(\vec{u}) \rangle -  \frac{|| \nabla I ||}{\sqrt{1 + || \nabla f ||^2}} \langle \nabla f, dN(\vec{v}) \rangle  \label{eq:shd_2nd_3}
\end{align}
\end{subequations}
\end{corollary}

Specifying further, let the tangent plane be frontal-parallel, i.e where the normal to the tangent plane is parallel to the view vector.  In this case, at the origin, vectors on the image plane are only translations of vectors on the tangent plane (in contrast to the general case, where there may be rotations, dilations, etc.)  The first fundamental form $G$ is now the identity matrix.  Thus, the shape operator reduces: $dN  = G^{-1} II  = II = n_3 H = H$.  In addition, the gradient $\nabla f = 0$.  Thus, the second term on the R.H.S of each equation is equal to 0.

\begin{corollary}
\begin{subequations}
\begin{align}
\frac{I_{vv}}{I} & =  - \vec{v} H^2 \vec{v} \\ 
\frac{I_{uu}}{I} & =   - \vec{u} H^2 \vec{u}  \\
\frac{I_{uv}}{I} & =  -\vec{v} H^2 \vec{u}
\end{align}
\end{subequations}
\end{corollary}

  In this important special case, we get the elegant result that the normalized second derivatives of intensity are proportional to the square of surface second derivatives.  This has been hypothesized in other work \cite{pentland84}.  The two-fold ambiguity, defined by the transformation $H \rightarrow -H$ emerges naturally. This is the well-explored \emph{concave/convex} ambiguity.  But there is more.

Without loss of generality, we assume that the Hessian is expressed in the basis $\{\vec{v}, \vec{u}\}$.  Then, the equations simplify once more:

\begin{subequations}
\begin{align}
\frac{I_{vv}}{I} & =  - (f_{xx}^2 + f_{xy}^2)  \label{eq:2nd_fp_b_1} \\
\frac{I_{uu}}{I} & =   - (f_{yy}^2 +  f_{xy}^2)    \label{eq:2nd_fp_b_2} \\
\frac{I_{uv}}{I} & =  -(f_{xx} + f_{yy}) f_{xy}  \label{eq:2nd_fp_b_3}
\end{align}
\end{subequations}

If we plot these three equations in 3-space defined by coordinate axes of $\{f_{xx}, f_{xy}, f_{yy}\}$, we see that the equations consist of two cylinders (situated perpendicular to each other) and a hyperbolic cylinder, which has a set of four intersection points.    Thus, we have a second 2-fold ambiguity, which corresponds to a saddle/ellipsoid type of ambiguity.  See Figure \ref{fig:4fold_ambig_FP}.  This is essentially unstudied in the SFS literature (but see \cite{Erens93}).

\begin{figure}[t]
\begin{center}
\includegraphics[width= 1 \linewidth]{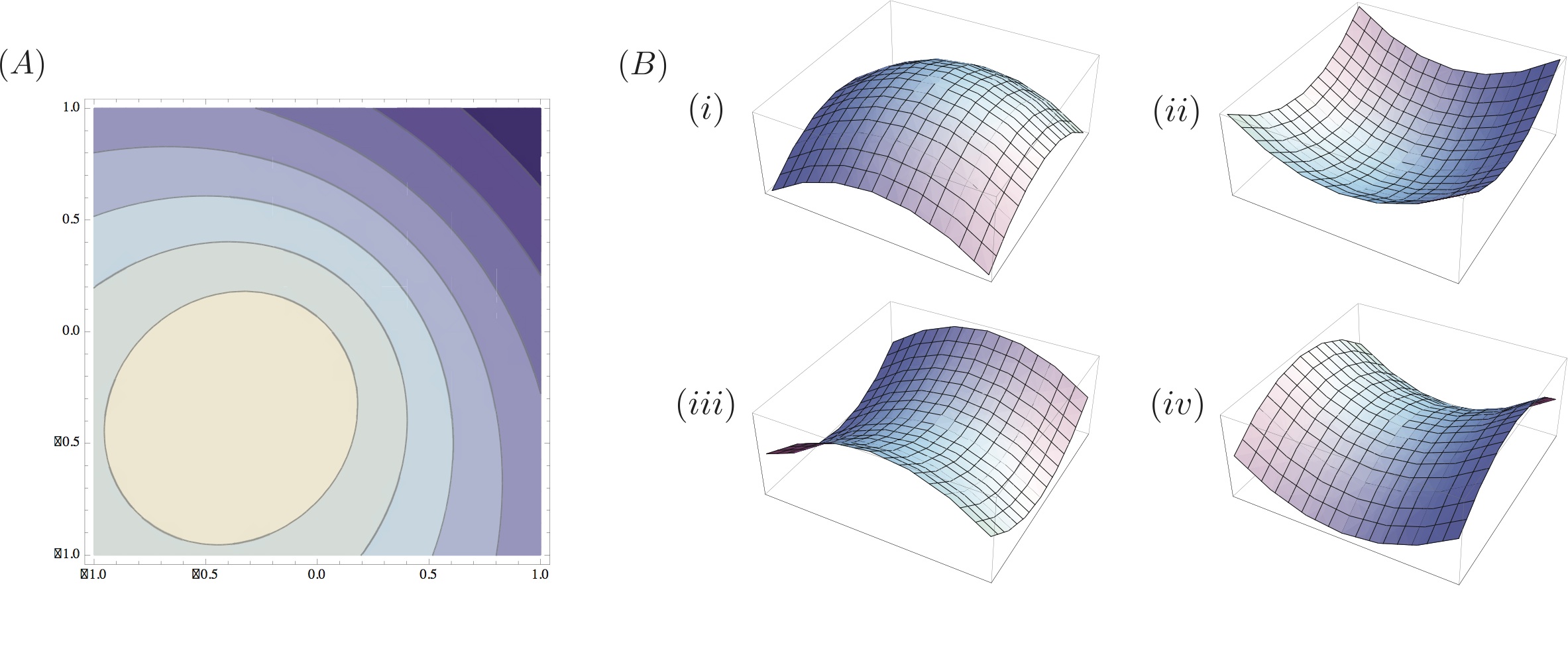}
\caption{For a frontal parallel second order Monge patch, we have exactly 4 surfaces available for each shading flow (A).  There is a concave/convex ambiguity in each row and a saddle/ellipsoid ambiguity in each column.}
\label{fig:4fold_ambig_FP}
\end{center}
\end{figure}


Note that we reduced the possible surfaces greatly by assuming the third order surface derivatives were 0 and that the tangent plane is frontal-parallel.  That is, we assumed knowledge of 4 parameters (on the third order derivatives) and 2 parameters (tangent plane).  In general, without these assumptions, the shading ambiguity space will be 6 dimensional.  In Section \ref{sec:2nd_order}, we will relax the condition of a frontal-parallel tangent plane.

\subsection{Second Order Assumption}
\label{sec:2nd_order}
 We turn back to the question posed in the first example, but with only the second order assumption on a surface patch.  That is, we again write $S = \{x, y, f(x, y) \}$ with $f(x, y) = a x + b y + c x^2 + d x y  + e y^2$.  This is a more complicated example than the previous two.  In this case, we work with the system  defined by the equations \ref{eq:shd_2nd_1},  \ref{eq:shd_2nd_2},  \ref{eq:shd_2nd_3}.

In the local second order patch, there are five parameters that need to be calculated: the two tangent plane orientation parameters and the three surface curvatures.  There are also an additional three parameters (two for the light source position and one for the albedo) that are involved in the image formation process.  If we consider the previous discussion, then we have 6 conditions $ \{ I, I_x, I_y, I_{xx}, I_{xy}, I_{yy} \}$ with 8 parameters.   Our equations factor out both the albedo and light source directions -- thus, we ignore $\{I, I_x, I_y\}$ and end up with 3 conditions $\{I_{xx}, I_{xy}, I_{yy} \}$ on 5 surface parameters.   It is reasonable to expect that we will have a two dimensional family of surfaces corresponding to each patch.  In addition, we expect that the family of surfaces may be parametrized by the light source position on the upper hemisphere or equivalently the two parameters of the tangent plane. 

This agrees with what we saw in Section \ref{sec:2nd_order_FP}.  Since we now have no assumption on the two parameters of the tangent plane, we must have a 2D ambiguity multiplied by any ambiguity we found in that section.

 We can apply the shading equations above to the Monge patch $f(x, y)$ to get three polynomial equations in $\{a, b, c, d, e\}$. We denote them as $g_i (c, d, e), i = 1, 2, 3$.  (For this analysis, we used \emph{Mathematica}). We will not display these polynomial equations here, as they are cumbersome but easy to replicate.  By choosing either the tangent plane coefficients (or equivalently the dominant light source direction), these 4th-order polynomial equations $\{g_i \}$ define the remaining coefficients $\{c, d, e\}$. 

For any chosen $\{a, b \}$ there may be either 0, 2, or 4 real roots of these polynomial equations.  Each root represents a corresponding second order Monge patch that would result in the same local shading flow.  Since the analysis is difficult, we have observed experimentally  that the polynomial system with a given shading flow and tangent plane choice will only have real roots outside a rectangle containing the origin.   That is, the region of tangent planes where there is no solutions is of the form $ \{ -x_0 \leq a \leq x_1 \cup -y_0 \leq b \leq y_1 \}$.   Of course, the rectangle's exact dimensions depend on the tangent plane choice.  Unfortunately, due to the complexity of this polynomial system, we are unable to state the exact relationship between the rectangle, the shading flow, and the choice of tangent plane.

However, given the shading flow and any single choice of the tangent plane $\{a, b\}$, we can solve the $\{g_i \}$ to calculate all the possible roots of the equation for that tangent plane choice, which is the necessary part for our goal of shape reconstruction. 
In Figure \ref{fig:2d_ambig}, we show an example of a shading flow that corresponds to each of the following surfaces, all with slightly different tangent planes.

\begin{figure}[t]
\begin{center}
\includegraphics[width= 1 \linewidth]{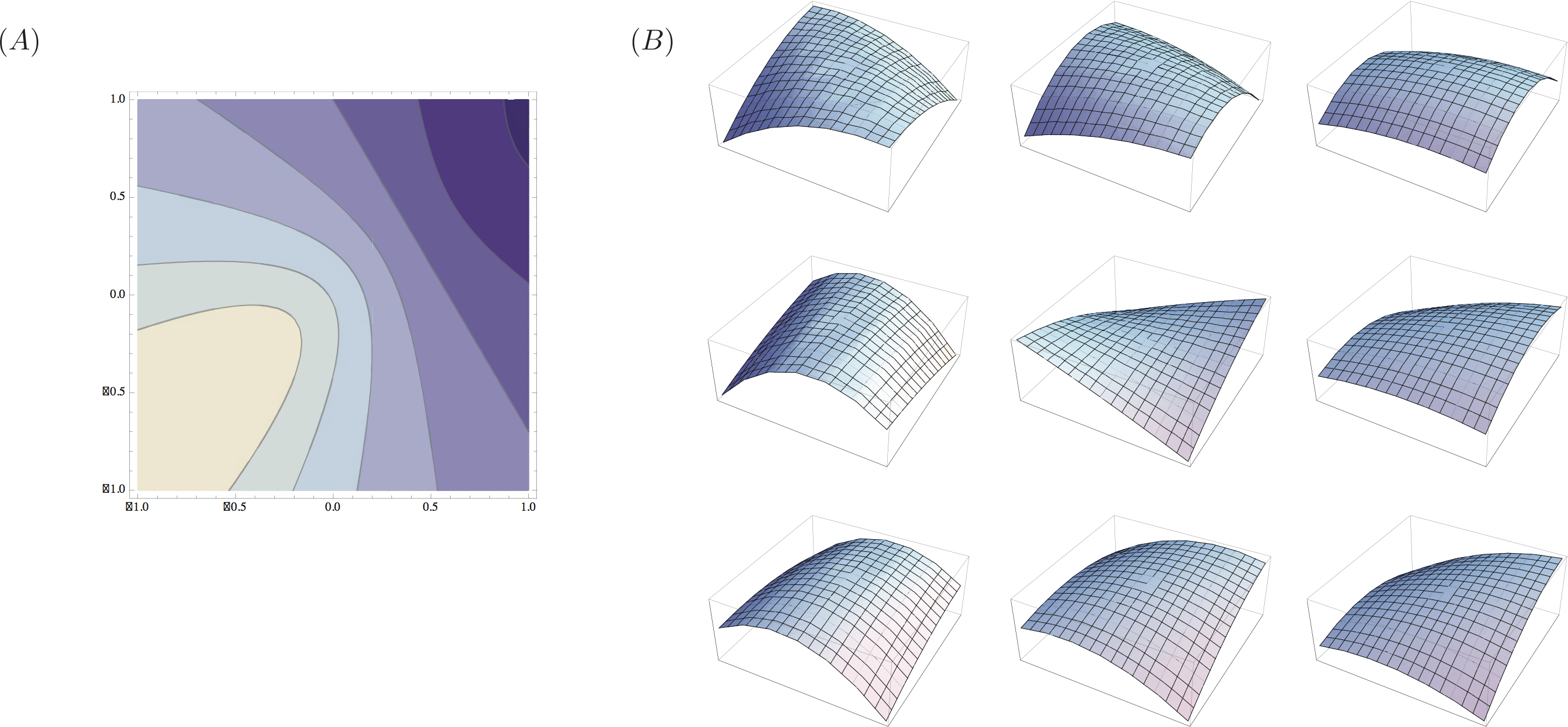}
\caption{Figure (A) shows a shading flow and Figure (B) shows 9 second order Monge patches, all with different tangent planes that are solutions to the polynomial equations and thus can result in the shading flow in Figure (A) when lit properly.}
\label{fig:2d_ambig}
\end{center}
\end{figure}

\subsection{Four-fold Ambiguity}

If there is to be a real solution to the polynomial system $\{g_i\}$ for a chosen $\{a_0, b_0 \}$, we must have either 2 or 4 solutions.  The reason for part of this ambiguity is due to the squared nature of the $||dN(\vec{v})||$ terms, just as we saw with the $H \rightarrow -H$ ambiguity in the frontal parallel case.  This leads to the standard concave/convex ambiguity pair.
The second pair of solutions (if it exists) is due to the saddle/spherical ambiguity as shown in Figure \ref{fig:4fold_ambig_FP} and Figure \ref{fig:4fold_ambig}.  We get one surface for each positive and negative pair of the principal curvatures: $\{k_1 > 0, k_2 > 0\},\{k_1 > 0, k_2 < 0\}, \{k_1 < 0, k_2 > 0\}, \{k_1 < 0, k_2 < 0\} $.  To see the complete ambiguity in movie form, please view the four GIFs in the Supplementary Information.  There is one GIF for each branch.

\begin{figure}[t]
\begin{center}
\includegraphics[width= 0.8 \linewidth]{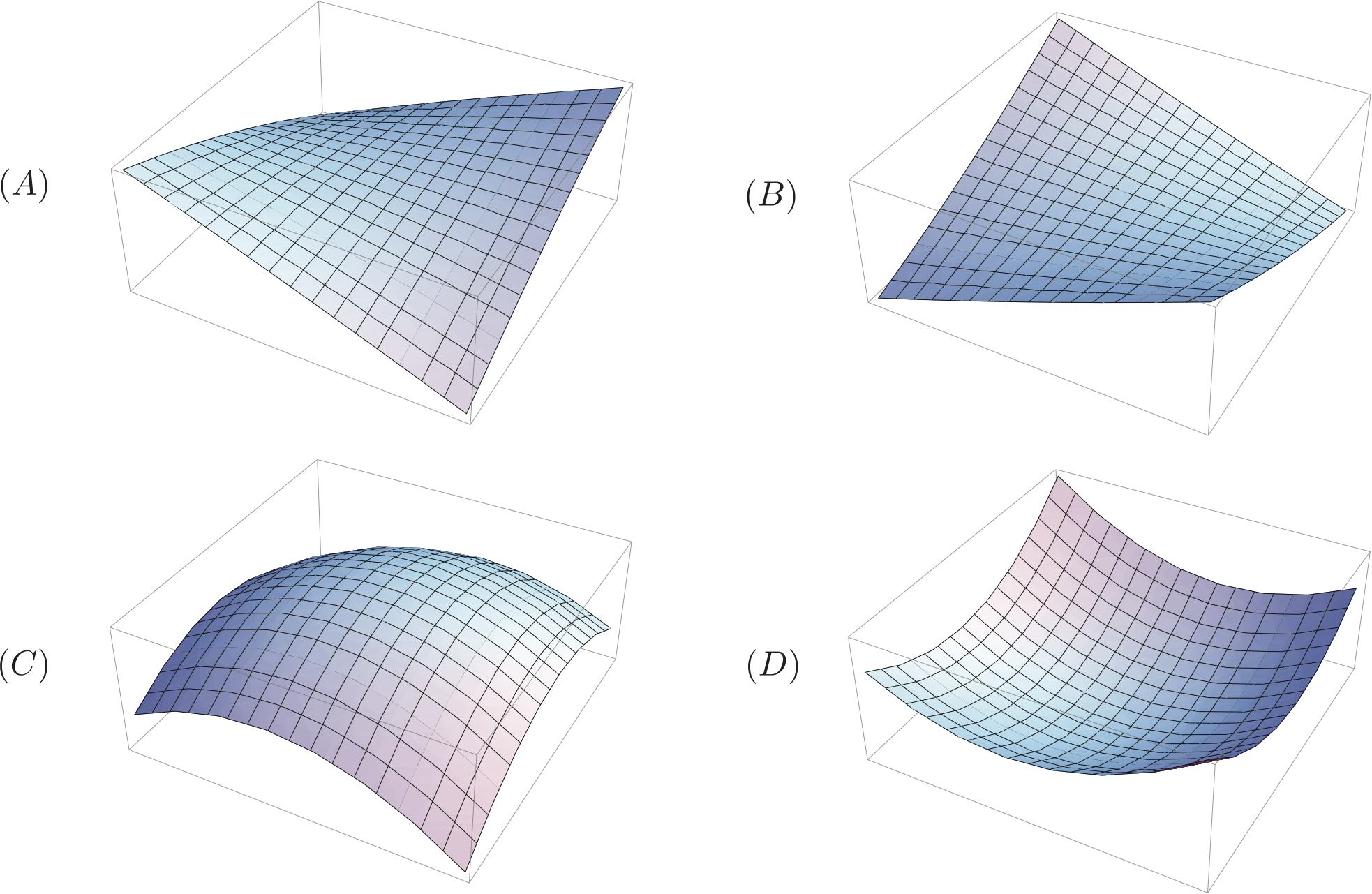}
\caption{Figure (A) is the surface corresponding to the frontal parallel tangent plane, as seen in the center of \ref{fig:2d_ambig} B).  In addition to the 2D ambiguity as illustrated in \ref{fig:2d_ambig} there can also be up to 4 different surfaces for a given tangent plane. Here,  (B), (C), (D), are the other surfaces corresponding to Figure (A).  Note that there are a pair of saddle surfaces and a pair of spherical surfaces.  The transformation taking (A) to (B) and (C) to (D) is the \emph{bas-relief} ambiguity.  In addition, there is a novel ambiguity between saddles and spherical shapes.}
\label{fig:4fold_ambig}
\end{center}
\end{figure}

To summarize, we can use these equations to recover all the second order surface information (up to the four-fold ambiguity) given any tangent plane orientation.  However, we do not have the tangent plane information \emph{a priori}.
Thus, even with the local second order assumption, we already must deal with at least a 2D shading ambiguity at every local patch.

For this reason, we believe the shape from shading problem can - and should- either be solved at certain points in the image (considered next) or should be combined with other means for obtaining tangent plane information.

\subsection{Ambiguity Reduction at Critical Points}

Much work has focused on the question: ``Where should one draw lines on a surface in order to give the best impression of the surface shape?"  Recently, Decarlo et al. \cite{decarlo03} have considered ``suggestive contours"  and Judd et al. have suggested apparent ridges \cite{Judd07}.  How do we decide which feature lines are ``better"? \cite{Cole09}  Why are certain curves so helpful in psychophysics?  We believe these questions can be answered by looking at the shading ambiguity on these feature lines. 

Consider the example of the highlight lines on a surface. We define them here as the points where the brightness gradient is a local maximum, i.e. $\nabla I = 0$.  For now, consider the generic case when the Gaussian curvature is not zero and so $H^{-1}$ is well-defined.  Then our equations \ref{eq:shd-theorem1},  \ref{eq:shd-theorem2} ,  \ref{eq:shd-theorem3} simplify:

\begin{subequations}
\begin{align}
I_{vv} & =  - (I) || dN(\vec{v}) ||^2 \label{eq:shd_cp_1}\\ 
I_{uu} & =   - (I) || dN(\vec{u}) ||^2 \label{eq:shd_cp_2}\\
I_{uv} & =  - (I) \langle dN(\vec{v}), dN(\vec{u}) \rangle  \label{eq:shd_cp_3}
\end{align}
\end{subequations}

This is quite similar to the second order, frontal parallel case, but we didn't need any assumptions!  Rather, we can only apply these equations to highlight points in the image.  This cursory analysis may explain why highlight lines are so effective at revealing surface shape psychophysically. We believe that understanding shading ambiguity can be a useful metric for deciding between different definitions of ``shape representing contours."  We can also go the other way; we can use the dimension of the shading ambiguity to define contours (sets of points) where the surface information is mathematically more restricted by the shading than at a generic point, key among these points are ridges \cite{kunsbergVSS12}.  Because of the complexity of this analysis, it will be treated in a companion paper.


%


\section{Discussion}

Given a local image patch $I$, assumed to be Lambertian, we can consider the continuous pixel intensity information in terms of derivatives of intensity.  At every point in the image, we have the information contained in $ \{ I, I_x, I_y, I_{xx}, I_{xy}, I_{yy} \}$.
We could consider more derivatives, but, in the limit, it is unnecessary.  The second order derivatives are the minimum order needed to remove the explicit light source variables.

Let us consider how to use each informational element.  The intensity $I$ at a point alone is useless in considering shape information (or even tangent plane orientation), as we don't know the albedo.   With unknown albedo, any angle between the unknown light sources and unknown normal is possible.
(Of course, we have information from intensity at \emph{several} nearby points, but we are considering that information to be contained in the derivatives of $I$.)

The amount of information in $I_x$ and $I_y$ is usually expressed as a brightness gradient.  Unfortunately, given any surface with second fundamental form $II$ of full rank, we can find a set of light source positions that will result in that brightness gradient.  Thus, $\{I_x, I_y\}$ does not give us any information about the surface unless we assume a prior on light source positions: the set of surfaces before and after we consider the brightness gradient is the same.  In terms of counting conditions, the brightness gradient gives us two conditions on the scene but having to include the explicit light source positions in the equations adds at least two more parameters.   However, once we have solved for our surface, we can use the brightness gradient to solve for the light source direction.

Finally, consider the second derivatives: $\{I_{xx}, I_{xy}, I_{yy} \}$.  As we have shown, these are the lowest order derivatives of intensity that can factored into components of surface shape  and image properties, with no explicit dependence on the light source.  Thus, our shape from shading reconstruction efforts will focus on the use of these equations in order to create a light-source invariant algorithm.

 \subsection{Local Surface Ambiguity}

Although these shading equations are nonlinear, it is still helpful to count the number of free parameters to get an idea of the ambiguity in each local patch.  Note that the derivatives $\{I_{uu}, I_{uv}, I_{vv}\}$ only depend on the local third-order Taylor approximation of the surface, which is described in nine coefficients.  However, we have only three equations available.  Roughly, we must have six dimensions of ambiguity.   

Although Lambertian shading does not restrict the set of possible surfaces to a single surface, additional information like specularities, texture flows, and boundaries may add in the additional restrictions needed to calculate a single surface.  For example, Koenderink's theorem \cite{Koenderink84} restricts surface curvatures at an apparent boundary.  In addition, the view vector must lie in the tangent plane at a boundary.  This provides three more conditions to restrict the local Taylor approximation, yet we are still several dimensions short.   In general, some assumptions on the surface are required in order to return a finite solution set.  




\section{Conclusion}

The differential invariants of surfaces are curvatures. Thus a natural
framework for formulating surface inferences is in terms of differential geometry. We here
propose such a framework, by lifting the image information to a vector field (the shading flow
field) and formulating the shape-from-shading problem on it. Our goal is to find those (surface,
light source) pairs that are consistent with a given shading flow. Working with simplifying
assumptions, we develop the basic transport machinery in closed form and calculate the full
family of solutions.

Isophote vector field changes in the image have two causes. First, the need for the vector field to ``stay on the surface" implies that there is a portion of orientation change due to the changing foreshortening of the surface tangent planes. This is important where the vector field approaches an occluding boundary: it must become parallel to this boundary, regardless of the light source(s) positions \cite{Lawlor09}. There, the foreshortening factor dominates the other factors that contribute to the orientation structure.

The other portion of change is due to the light source field ``treating" different local tangent planes differently due to its projection onto them. However, because the light source does not change in the extrinsic $\mathbb{R}^3$ frame, this second portion of change can also be understood through purely surface properties. That is, the local light source change can be related solely to the surface curvatures ($dN$), modified by the image intensity. Thus, both portions of the shading flow change are, in the end, only dependent on surface properties and not dependent on the placement in $\mathbb{R}^3$ of the light source.

Finally, we close with a neurobiological point. It is known that the higher visual areas are
selective for surface properties, including their curvatures \cite{Connor08}. It is also known that many
different forms of orientation images, such as oriented texture noise and glossy patterns (see
references in \cite{fleming}) are perceived as surfaces. To our knowledge the calculations here are the
first example of how this inference might take place from the shading flow to surfaces. It
thus serves as a common ÒlanguageÓ for formulating feedback,  but also illustrates a need for additional information. Perhaps the importance of highlights, or texture elements, could select from the ambiguous family.  While we have extended the mathematics of the shape from shading flow problem to much more general situations than the examples treated here, much remains to be done with the differential geometric approach.

\newpage
\begin{appendices}
\appendix
\section{\\Proofs for Shading Equations $I_{uu}$ and $I_{uv}$} \label{App:AppendixA}

The proofs for the shading equations \ref{eq:shd-theorem2} and \ref{eq:shd-theorem3} are analogous to the proof for the first shading equation  \ref{eq:shd-theorem1}.  Rather than repeat the analysis almost verbatim, we instead describe where substitutions need to be made.

\subsection{For $I_{uv}$}
Here, we take the directional derivative of the constraint $\langle \vec{l_t}, dN(\vec{v}) \rangle$ in the $\vec{u}$ direction.
Thus, we need to modify 3.5a so that we take the directional derivative with respect to $\vec{u}$ rather than $\vec{v}$.
From then on, every time we see a $\vec{v}$ as a direction in which to take a derivative, we will instead have $\vec{u}$.  The analysis for $I_{uv}$ then follows the $I_{vv}$ case exactly except in one place:  

The first term in 3.14d (that was  $\vec{T_{l_t}} II \vec{v}$) is now $\vec{T_{l_t}} II \vec{u}$ and thus contributes a term rather than 0.   However this simplifies (using the Christoffel symbols $\Gamma^k_{ij})$ to:

\begin{subequations}
\begin{align}
\vec{T_{l_t}} & = -  \begin{bmatrix} \sum_{i, j} \Gamma^1_{i j}\vec{l_t}^i u^j & \sum_{i, j} \Gamma^2_{i j}\vec{l_t}^i u^j  \end{bmatrix} \\
 & = - \begin{bmatrix} \frac{f_x (\vec{l_t}^T II \vec{u})}{\sqrt{1 + f_x^2 + f_y^2}} &   \frac{f_y (\vec{l_t}^T II \vec{u})}{\sqrt{1 + f_x^2 + f_y^2}
}\end{bmatrix}\\
& = - \vec{l_t}^T II \vec{u} \frac{ \begin{bmatrix} f_x & f_y \end{bmatrix}}{\sqrt{1 + f_x^2 + f_y^2}} \\
& = - || \nabla I || \frac{ \nabla f}{\sqrt{1 + || \nabla f ||^2}}
\end{align}
\end{subequations}

Thus,

\begin{subequations}
\begin{align}
\vec{T_{l_t}} II \vec{u} & = - || \nabla I || \frac{ (\nabla f)^T II \vec{u} }{\sqrt{1 + || \nabla f ||^2}} \\
& = -   \frac{|| \nabla I|}{\sqrt{1 + || \nabla f ||^2}} \langle \nabla f,  dN( \vec{u}) \rangle
\end{align}
\end{subequations}

And this is precisely the extra term included in \ref{eq:shd-theorem2}.

\subsection{For $I_{vv}$}
Here, we take the directional derivative of the constraint $\langle \vec{l_t}, dN(\vec{u}) \rangle$ in the $\vec{u}$ direction.  Thus, the only difference between this proof and the previous one for $I_{uv}$ is that the terms containing $dN(\vec{v})$ in the previous proof are now $dN(\vec{u})$.  This leads to only minor change.

In 3.9b, the second term is now $-\vec{u}[n_3] G^{-1} H (\vec{u})$.  In the previous proofs, the corresponding term enters an inner product with $\vec{l_t}$ (which is equal to 0) and thus contributes nothing.  Here, we must calculate it separately.

Straightforward calculation yields:

\begin{subequations}
\begin{align}
\vec{u} [n_3] & =  \frac{\langle \nabla f, dN(\vec{u}) \rangle}{\sqrt{1 + ||\nabla f ||^2}} \\
\end{align}
\end{subequations}

Thus, 
\begin{subequations}
\begin{align}
\langle \vec{l_t} , -\vec{u}[n_3] G^{-1} H (\vec{u}) \rangle & =  - \frac{\langle \nabla f, dN(\vec{u}) \rangle}{\sqrt{1 + ||\nabla f ||^2}} \langle \vec{l_t}, dN(\vec{u}) \rangle  \\
& = - ||\nabla I || \frac{\langle \nabla f, dN(\vec{u}) \rangle}{\sqrt{1 + ||\nabla f ||^2}} 
\end{align}
\end{subequations}

Therefore, the sum of the two extra terms we get when applying the method for $I_{uu}$ is solely 

$$- 2 ||\nabla I || \frac{\langle \nabla f, dN(\vec{u}) \rangle}{\sqrt{1 + ||\nabla f ||^2}} $$

Adding in these respective terms into the formulas for $I_{uv}$ and $I_{uu}$  and changing the appropriate $\vec{v}$ to $\vec{u}$ give the second order shading equations stated above.

\section{\\PDE Equations} \label{App:AppendixB}
For completeness, we add the shading equations in PDE fashion, without the differential geometric notation.

\begin{subequations}
\begin{align}
0 & = f_{xx} \\
& + \left(\frac{(1 + f_x^2) f_{xy}^2  - 2 f_x f_y f_{xy} f_{xx} + (1 + f_y^2) f_{xx}^2}{(1 + f_x^2 + f_y^2)^2}\right) I \\
& + 2 I_x \frac{f_x f_{xx} + f_y f_{xy}}{{1 + f_x^2 + f_y^2}} \\
& - \big(I_x (f_{yy} f_{xxx} - f_{xy} f_{xxy}) + I_y (-f_{xy} f_{xxx} + f_{xx} f_{xxy})\big)
\end{align}
\end{subequations}

\begin{subequations}
\begin{align}
0 & = f_{yy} \\
& +  \left(\frac{(1 + f_x^2) f_{yy}^2  - 2 f_x f_y f_{xy} f_{yy} + (1 + f_y^2) f_{xy}^2}{(1 + f_x^2 + f_y^2)^2}\right) I \\
& + 2 I_y \frac{f_x f_{xy} + f_y f_{yy}}{{1 + f_x^2 + f_y^2}} \\
& - \big(I_x (f_{yy} f_{xyy} - f_{xy} f_{yyy}) + I_y (-f_{xy} f_{xyy} + f_{xx} f_{yyy})\big)
\end{align}
\end{subequations}

\begin{subequations}
\begin{align}
0 & = f_{xy} \\
& + \left(\frac{f_{xy} (f_{xx} + f_{yy} + f_y^2 f_{xx} + f_x^2 f_{yy})  -  f_x f_y (f_{xy}^2 +  f_{xx} f_{yy}}{(1 + f_x^2 + f_y^2)^2}\right) I\\
& + \frac{f_x I_y f_{xx} + (f_x I_x + f_y I_y) f_{xy} + f_y I_x f_{yy}}{{1 + f_x^2 + f_y^2}} \\
& - \big(I_x (f_{yy} f_{xxy} - f_{xy} f_{xyy}) + I_y (-f_{xy} f_{xxy} + f_{xx} f_{xyy})\big)
\end{align}
\end{subequations}

\end{appendices}

\newpage

{\small
\bibliographystyle{siam}
\bibliography{Kunsberg_Shading_Ambiguity}

\begin{thebibliography}{10}

\bibitem{Barron13}
{\sc J.~Barron and J.~Malik}, {\em Shape, illumnation, and reflectance from
  shading}, Technical Report,  (2013).

\bibitem{Bel98}
{\sc P.~Belhumeur and D.~Kriegman}, {\em What is the {S}et of {I}mages of an
  {O}bject under {A}ll {P}ossible {I}llumination {C}onditions?}, International
  Journal of Computer Vision, 28 (1998), pp.~1--16.

\bibitem{Bel99}
{\sc P.~Belhumeur, D.~Kriegman, and A.~Yuille}, {\em The {B}as-{R}elief
  {A}mbiguity}, International Journal of Computer Vision, 35 (1999),
  pp.~33--44.

\bibitem{zucker}
{\sc P.~Breton and S.W. Zucker}, {\em Shadows and {S}hading {F}low {F}ields},
  Proc. IEEE Conf. on Computer Vision and Pattern Recognition - CVPR '96,
  (1996), pp.~782 -- 789.

\bibitem{horn85}
{\sc M.J. Brooks and B.K.P Horn}, {\em Shape and {S}ource from {S}hading}, in
  Proceedings of International Joint Conference on Artificial Intelligence,
  1985, pp.~932--936.

\bibitem{Cole09}
{\sc F.~Cole, K.~Sanik, D.~DeCarlo, A.~Finkelstein, T.~Funkhouser,
  S.~Rusinkiewicz, and M.~Singh}, {\em How well do line drawings depict
  shape?}, ACM Trans. Graph., 28 (2009).

\bibitem{decarlo03}
{\sc D.~DeCarlo, A.~Finkelstein, S.~Rusinkiewicz, and A.~Santella}, {\em
  Suggestive contours for conveying shape}, SIGGRAPH,  (2003), pp.~848--855.

\bibitem{Dieft81}
{\sc P.~Dieft and J.~Sylvester}, {\em Some remarks on the shape-from-shading
  problem in computer vision}, Journal of Mathematical Analysis and
  Applications, 84 (1981), pp.~235--248.

\bibitem{docarmo}
{\sc M.P. Docarmo}, {\em Differential {G}eometry of {C}urves and {S}urfaces},
  Prentice-Hall Inc., Upper Saddle River, New Jersey, 1976.

\bibitem{Dodson91}
{\sc C.T.J Dodson and T.~Potson}, {\em Tensor {G}eometry}, Springer Press,
  Berlin Heidelberg, 1991.

\bibitem{Dupuis94}
{\sc P.~Dupuis and J.~Oliensis}, {\em An optimal control formulation and
  related numerical methods for a problem in shape reconstruction}, The annals
  of Applied Probability, 4 (1994), pp.~287--346.

\bibitem{Erens93}
{\sc R.. Erens, A.~Kappers, and J.J. Koenderink}, {\em Perception of local
  shape from shading}, Perception \& Psychophysics, 54 (1993), pp.~145--156.

\bibitem{fleming}
{\sc R.~Fleming, D.~Holtmann-Rice, and H.~Bulthoff}, {\em Estimation of {3D}
  {S}hape from {I}mage {O}rientations}, Proceedings of the National Academy of
  Sciences, 108 (2011).

\bibitem{Forsyth11}
{\sc D.A Forsyth}, {\em Variable-{S}ource {S}hading {A}nalysis}, International
  Journal of Computer Vision, 91 (2011), pp.~280--302.

\bibitem{Freeman94}
{\sc W.T Freeman}, {\em The generic viewpoint assumption in a framework for
  visual perception}, Nature, 368 (1994), pp.~542--545.

\bibitem{Garding95}
{\sc J.~Garding}, {\em Surface orientation and curvature from differential
  texture distortion}, Proc. 5th International Conference on Computer Vision,
  (1995).

\bibitem{Brooks86}
{\sc B.K.P Horn and M.~Brooks}, {\em The variational approach to shape from
  shading}, Computer Vision, Graphics, and Image Processing., 33 (1986),
  pp.~174--208.

\bibitem{Hubel88}
{\sc D.~Hubel}, {\em Eye, {B}rain, and {V}ision}, Scientific American Library,
  1988.

\bibitem{Ikeuchi81}
{\sc K.~Ikeuchi and B.K.P Horn}, {\em Numerical shape from shading and
  occluding boundaries}, Artificial Intelligence,  (1981).

\bibitem{Judd07}
{\sc T.~Judd, F.~Durand, and E.H Adelson}, {\em Apparent ridges for line
  drawing}, ACM Trans. Graph., 26 (2007), p.~19.

\bibitem{Koenderink84}
{\sc J.J. Koenderink}, {\em What does the {O}ccluding {C}ontour tell us about
  {S}olid {S}hape?}, Perception, 13 (1984), pp.~321 --330.

\bibitem{Koenderink90}
\leavevmode\vrule height 2pt depth -1.6pt width 23pt, {\em Solid Shape}, The
  MIT Press, Cambridge, Massachusetts, 1990.

\bibitem{Koenderink80}
{\sc J.J. Koenderink and A.J. Van~Doorn}, {\em Photometric invariants related
  to solid shape}, Optica Acta, 27 (1980), pp.~981--996.

\bibitem{Kunsberg12}
{\sc B.~Kunsberg and S.W. Zucker}, {\em The differential geometry of shape from
  shading: Biology reveals curvature structure}, The 8th IEEE Workshop on
  Percept. Org. in Comp. Vision,  (2012).

\bibitem{kunsbergVSS12}
\leavevmode\vrule height 2pt depth -1.6pt width 23pt, {\em Shape-from-shading
  and cortical computation: a new formulation}, Journal of Vision, 12 (2012).

\bibitem{Lawlor09}
{\sc M.~Lawlor, D.~Holtmann-Rice, P.~Huggins, O.~Ben-Shahar, and S.W. Zucker},
  {\em Boundaries, shading, and border ownership: A cusp at their interaction},
  Journal of Physiology - Paris, 103 (2009), pp.~18--36.

\bibitem{Lions93}
{\sc P.-L. Lions, E.~Rouy, and A.~Tourin}, {\em Shape-from-shading, viscosity
  solutions and edges}, Numer. Math., 64 (1993), pp.~323--353.

\bibitem{mach}
{\sc E.~Mach}, {\em On the {P}hysiological {E}ffect of {S}patially
  {D}istributed {L}ight {S}timuli {(Transl, F. Ratliff)}}, in Mach Bands:
  Quantitative studies on neural networks in the retina, Holden Day, San
  Francisco, 1965.

\bibitem{Mingolla86}
{\sc E.~Mingolla and J.T. Todd}, {\em Perception of solid shape from shading},
  Biological Cybernetics, 53 (1986), pp.~137--151.

\bibitem{Oliensis91}
{\sc J.~Oliensis}, {\em Uniqueness in shape from shading}, IJCV, 2 (1991),
  pp.~75--104.

\bibitem{pentland84}
{\sc A.~Pentland}, {\em Local {S}hading {A}nalysis}, IEEE Transactions on
  Pattern Analysis and Machine Intelligence, PAMI-6 (1984).

\bibitem{Prados04}
{\sc E.~Prados and O.~Faugeras}, {\em Unifying approaches and removing
  unrealistic assumptions in shape from shading: Mathematics can help},
  Proceedings of the 8th European Conference on Computer Vision,  (2004),
  pp.~141--154.

\bibitem{faugeras}
{\sc E.~Prados and O~Faugeras}, {\em Shape from {S}hading}, in Handbook of
  Mathematical Models in Computer Vision, N.~Paragios, Y.~Chen, and
  O.~Faugeras, eds., Springer Science, New York, NY, 2006, pp.~375 -- 388.

\bibitem{Ramachandran88}
{\sc V.S. Ramachandran}, {\em Perceiving shape from shading}, Scientific
  American, 259 (1988), pp.~76--83.

\bibitem{Rouy92}
{\sc E.~Rouy and A.~Tourin}, {\em A viscosity solutions approach to
  shape-from-shading}, SIAM Journal of Numerical Analysis, 29 (1992),
  pp.~867--884.

\bibitem{Shapley85}
{\sc R.~Shapley and J.~Gordon}, {\em Nonlinearity in the perception of form},
  Perception and Psychophysics, 37 (1985), pp.~84 -- 88.

\bibitem{Sun12}
{\sc P.~Sun and A.J. Schofield}, {\em Two operational modes in the perception
  of shape from shading revealed by the effects of edge information in slant
  settings}, Journal of Vision, 12 (2012).

\bibitem{Wagemans10}
{\sc J.~Wagemans, A.J. Van~Doorn, and J.J Koenderink}, {\em The shading cue in
  context}, i-Perception, 1 (2010), pp.~159--178.

\bibitem{Connor08}
{\sc Y.~Yamane, E.T. Carlson, K.C. Bowman, Z.~Wang, and C.E. Connor}, {\em A
  {N}eural {C}ode for {T}hree-{D}imensional {O}bject {S}hape in {M}acaque
  {I}nferotemporal {C}ortex}, Nature Neuroscience: Published Online,  (2008).

\bibitem{zhang99}
{\sc R.~Zhang, T.~Ping-Song, J.~Cryer, and M.~Shah}, {\em Shape from shading: A
  survey}, IEEE Transactions on Pattern Analysis and Machine Intelligence, 21
  (1999).

\end{thebibliography}
}

\end{document}